\documentclass[11pt]{article}

\usepackage[final]{acl}

\usepackage{times}
\usepackage{latexsym}
\usepackage{amsmath}
\usepackage{amsthm}
\usepackage{amssymb}
\usepackage{algorithm}
\usepackage{algpseudocode}
\usepackage{amsmath}
\usepackage{hyperref}
\usepackage{graphicx}
\usepackage{subcaption}  
\usepackage{booktabs}
\usepackage{bm}
\usepackage{amsmath, booktabs, enumitem}
\usepackage{mathtools}
\usepackage{microtype}

\newtheorem{assumption}{Assumption}
\newtheorem{definition}{Definition}
\newtheorem{lemma}{Lemma}
\newtheorem{theorem}{Theorem}
\newtheorem{corollary}{Corollary}
\newtheorem{remark}{Remark}

\usepackage[T1]{fontenc}

\usepackage[utf8]{inputenc}
\usepackage{booktabs}
\usepackage{microtype}
\usepackage{wrapfig}
\usepackage{inconsolata}

\usepackage{graphicx}
\usepackage{subcaption}

\title{LAARA: Layer-Aware Adaptive Rank Allocation for Parameter-Efficient Fine-Tuning}

\author{
  Ashutosh Tripathi$^{1}$, \quad
  Surya Deep Singh$^{1}$, \quad
  Pranab Sahoo$^{2}$, \quad
  Sriparna Saha$^{2}$ \\
  \\
  $^{1}$Independent Researcher \quad
  $^{2}$Indian Institute of Technology, Patna \\
  \\
  \small
  \texttt{ashutoshtripathi191@gmail.com}, \texttt{suryadeepsingh95@gmail.com}, \texttt{pranab\_2021cs25@iitp.ac.in}, \texttt{sriparna@iitp.ac.in}
}

\begin{document}
\maketitle
\begin{abstract}
Low-Rank Adaptation is widely used for parameter-efficient fine-tuning, yet existing methods typically assign the same adapter rank to every transformer layer despite their heterogeneous adaptation requirements. In this work, we show theoretically and empirically that uniform rank allocation is fundamentally suboptimal. Motivated by this observation, we propose LAARA (Layer Aware Adaptive Rank Allocation framework), a search-free framework that dynamically allocates ranks using lightweight diagonal Fisher estimates computed during training. LAARA combines projection-wise normalization, logarithmic compression, blended adapter importance estimation, and a vote-to-change dampening mechanism to produce stable and efficient rank adaptation. Experiments on GLUE and MathInstruct benchmark demonstrate that LAARA consistently matches or outperforms popular state of the art approaches such as LoRA, AdaLoRA, DyLoRA, and Bitfit while using significantly fewer trainable parameters. Our results show that Fisher-guided rank allocation provides a principled and effective foundation for adaptive parameter-efficient fine-tuning. The code is publicly available at:
\url{https://anonymous.4open.science/r/LAARA-D305/LAARA.py}
\end{abstract}


\section{Introduction}
\label{sec:introduction}
\textls[-20]{
Large-scale foundation models have transformed natural language processing
and computer vision~\citep{achiam2023gpt,kirillov2023segment,sahoo2025feddual,sahoo2026harnessing,11261399,sahoo2024fedmrldataheterogeneityaware,10681112}, yet
full fine-tuning remains prohibitive in memory, compute, and storage.
Parameter-Efficient Fine-Tuning (PEFT) methods address this by updating
only a small subset of parameters while keeping pretrained weights frozen.
Low-Rank Adaptation (LoRA)~\citep{hu2022lora} is among the most widely
adopted, decomposing the task-specific weight update as:
\begin{equation}
  \Delta W = BA, \quad B \in \mathbb{R}^{d_{out} \times r},\;
  A \in \mathbb{R}^{r \times d_{in}},
  \label{eq:lora}
\end{equation}
where $r \ll \min(d_{out}, d_{in})$, with no added inference latency.
Despite its simplicity, as we show in this paper, the uniform rank
assumption underlying standard LoRA is provably suboptimal and a
principled, search-free remedy exists.

\paragraph{The problem of uniform rank and existing adaptive methods.}
Standard LoRA applies a single global rank across all weight matrices,
ignoring the empirically well-established heterogeneity in layer-wise
adaptation behaviour: some layers are far more
sensitive to downstream tasks than others, making any fixed rank
simultaneously suboptimal for all layers. Setting $r$ too high wastes
capacity and risks overfitting; setting it too low starves critical
layers of expressive power. Existing adaptive methods address this
partially but incompletely: AdaLoRA~\citep{zhang2023adalora} prunes
singular values during training but requires costly SVD decomposition
and sensitive regularization tuning; AutoLoRA~\citep{zhou2024autolora}
frames rank selection as architecture search, incurring substantial
overhead; DyLoRA~\citep{valipour2022dylora} trains across multiple
ranks simultaneously without principled per-layer assignment; and
SoRA~\citep{ding2023sparse} relies on heuristic sparsity thresholds.
Critically, none connects rank allocation to the statistical importance
of each layer as measured by loss-landscape curvature, a gap that
leaves parameter budgets allocated by approximation rather than
principle. More discussion of related work is in the
Appendix~\ref{sec:related_work}.

\paragraph{Fisher information as a principled criterion.}
The FIM measures how sensitively the model's predictive distribution
changes with respect to parameter perturbations~\citep{martens2015optimizing,
kirkpatrick2017ewc}, providing a theoretically grounded, search-free
signal for rank allocation. Layers with large Fisher information demand
higher-rank adapters to capture task-specific signals, while layers with
near-zero Fisher information can safely be assigned very small ranks.
Complementarily, the eigenspectrum of the layer-wise FIM encodes the
intrinsic dimensionality of task-relevant parameter variation, directly
motivating a data-driven rank assignment based on effective Fisher rank.
While Fisher information has been used in PEFT before,
\citet{kirkpatrick2017ewc} applied it to prevent forgetting in continual
learning without connecting it to adapter capacity;
\citet{xue2025fishtuningenhancingpeftmethods} used it for sparse parameter
selection but applied a uniform threshold across layers, ignoring rank
heterogeneity; and FLoE~\citep{wang2025floe} combines Fisher-guided layer
selection with Bayesian search, reintroducing the optimization overhead
that a principled Fisher criterion should eliminate. LAARA is the first
to derive layer-wise ranks directly from Fisher estimates without
search or auxiliary objectives.

\paragraph{Empirical motivation.}
Computing the Fisher trace exactly requires expensive Hessian estimation,
but a tractable alternative exists: by the empirical Fisher
approximation~\citep{martens2015optimizing}, the expected squared gradient
Frobenius norm is a consistent estimator of the Fisher trace (formalized
in Theorem~2, Part~2), meaning gradient norms recorded during standard
LoRA fine-tuning serve as a lightweight, training-time proxy for
layer-wise Fisher importance without any additional forward passes or
hyperparameter search. To validate this connection empirically, we
analyze fine-tuning dynamics of DeBERTa-v3 on GLUE RTE and CoLA,
surfacing three consistent observations that directly motivate the
design of LAARA (full details in Appendix~\ref{sec:impl_details}).


\textbf{Observation 1:} Gradient norms increase monotonically with layer
depth, primarily driven by \texttt{lora\_A} ($\rho = 0.972$,
$p = 1.29\times10^{-7}$), while \texttt{lora\_B} exhibits a weaker trend
($\rho = 0.692$, $p = 1.26\times10^{-2}$) with non-monotonic fluctuations
at middle layers (Figures~\ref{fig1} and~\ref{fig1q}). This asymmetry
arises because \texttt{lora\_A} directly encodes input-projected directions
and is a more faithful proxy for layer-wise curvature, motivating the
blended importance formulation of Eq.~\ref{eq:blend}.

\textbf{Observation 2:} Fisher traces are highly heterogeneous and stable
during training (Figures~\ref{fig2} and~\ref{fig1qe}). Layer~11
accumulates a Fisher trace nearly two orders of magnitude larger than
Layers~0 and~6, confirming that uniform rank is provably suboptimal and
that this heterogeneity is not dataset-specific. The layer-importance
ordering stabilizes after epoch~3, meaning a single lightweight
calibration pass suffices to estimate per-layer importance reliably.

\textbf{Observation 3:} Effective rank of weight updates decreases with
depth (Figures~\ref{fig3} and~\ref{fig3q}), falling from $\approx\!6$ at
Layer~0 to $\approx\!2$ at Layer~11. Early layers spread their budget
across many directions while late layers concentrate it in a few dominant
eigendirections — confirming that uniform rank systematically misallocates
capacity across layers whose intrinsic adaptation dimensionality differs by
a factor of $\sim\!3\times$.

Motivated by these findings, we propose \textbf{LAARA}, a
\textbf{L}ayer-\textbf{A}ware \textbf{A}daptive \textbf{R}ank
\textbf{A}llocation framework that derives layer-wise ranks directly
from a lightweight diagonal FIM approximation of LoRA gradients
without iterative search, auxiliary objectives, or meta-learners.
Rank assignments are stabilized through projection-wise normalization,
logarithmic importance compression, a blended adapter importance
formulation (Eq.~\ref{eq:blend}), and a vote-to-change dampening
mechanism that prevents rank oscillation during training.



\begin{figure}[ht]
    \centering
    
    \begin{subfigure}{0.48\textwidth}
        \centering
        \includegraphics[width=\linewidth]{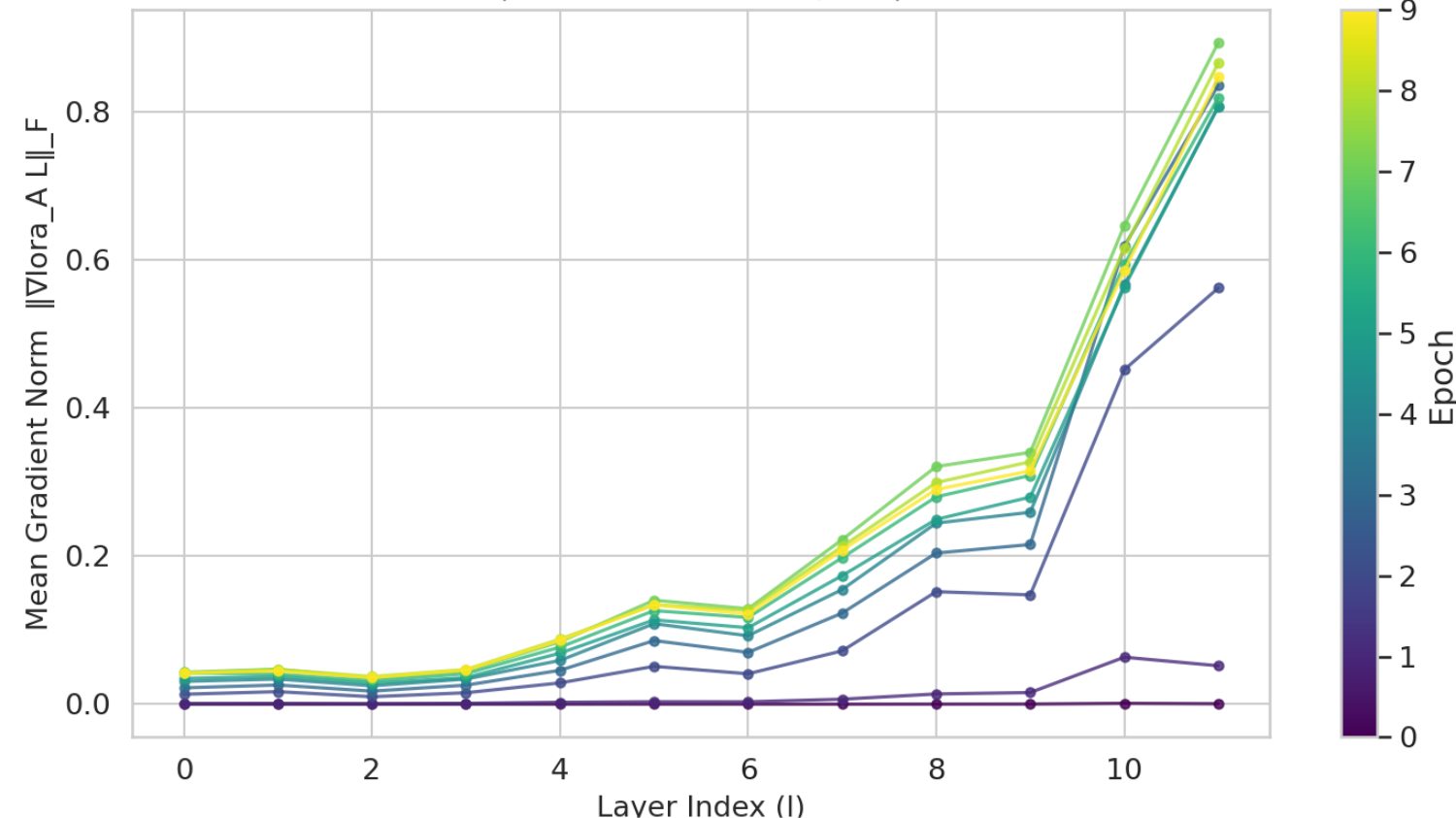}
        \caption{}
    \end{subfigure}
    \hfill
    \begin{subfigure}{0.48\textwidth}
        \centering
        \includegraphics[width=\linewidth]{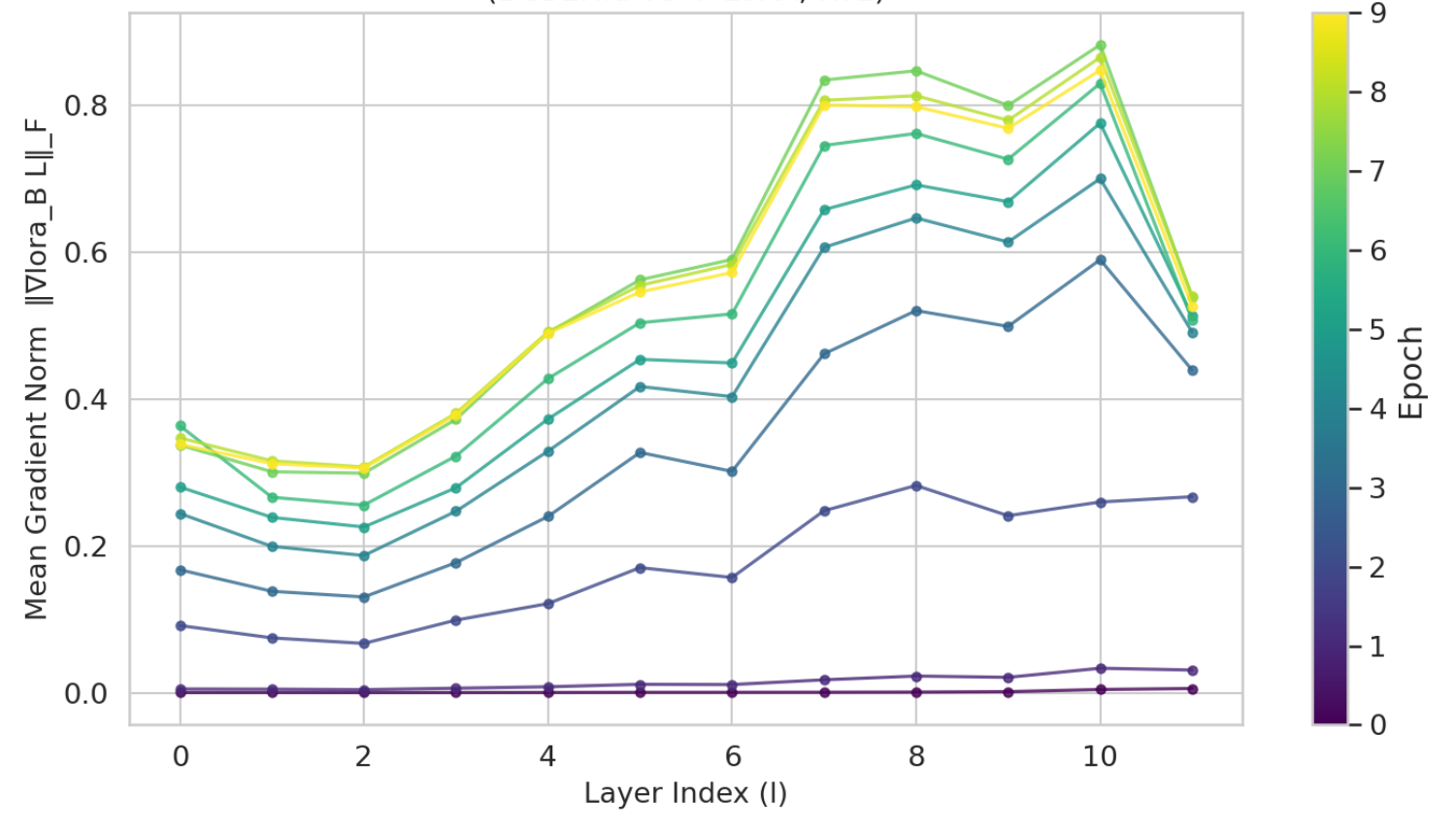}
        \caption{}
    \end{subfigure}
    
    \caption{Mean gradient norms vs layer index for Lora A (a) and Lora B (b) on RTE dataset.}
    \label{fig1}
\end{figure}

\begin{figure}[ht]
    \centering
    
    \begin{subfigure}{0.48\textwidth}
        \centering
        \includegraphics[width=\linewidth]{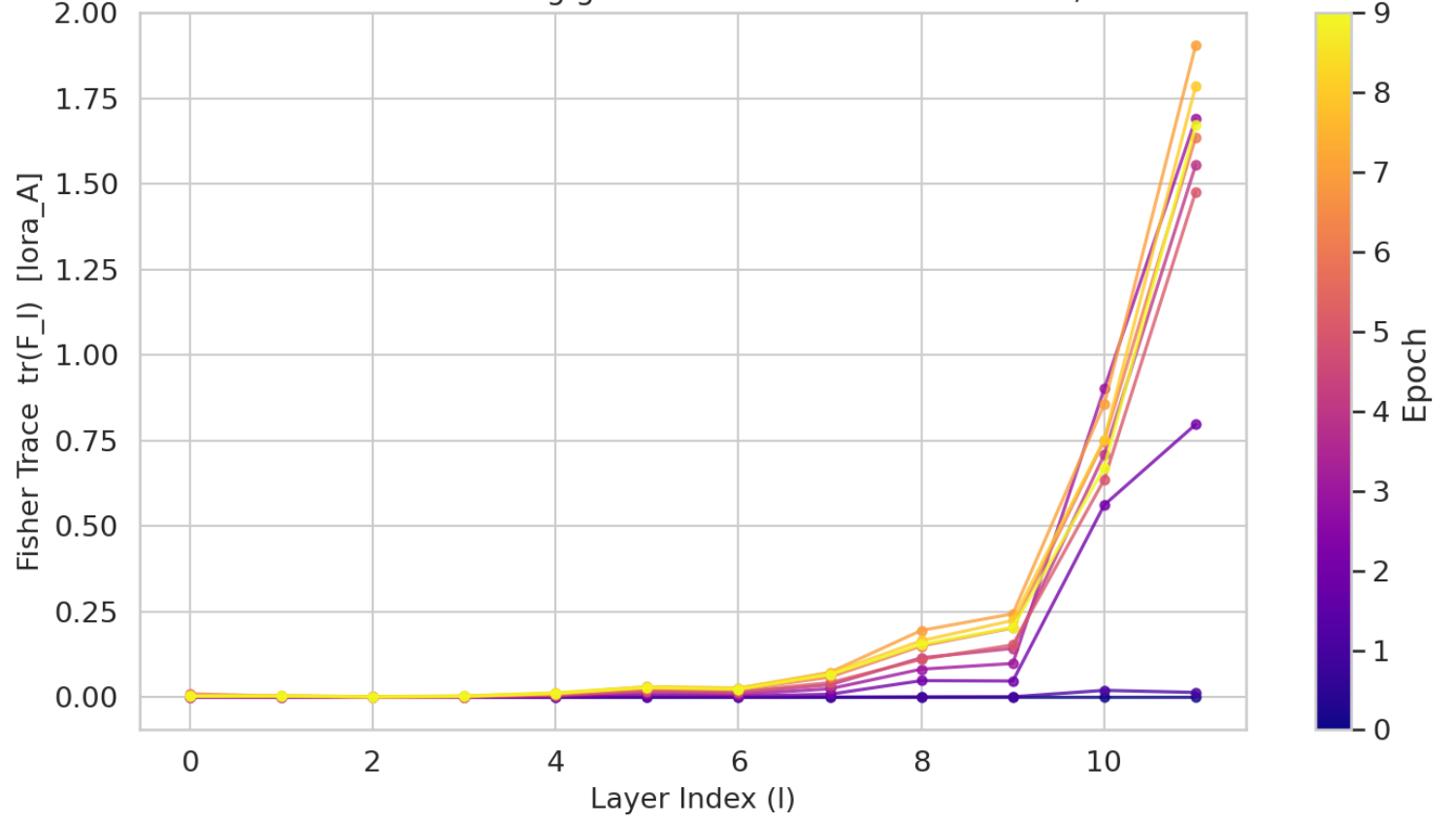}
        \caption{}
    \end{subfigure}
    \hfill
    \begin{subfigure}{0.48\textwidth}
        \centering
        \includegraphics[width=\linewidth]{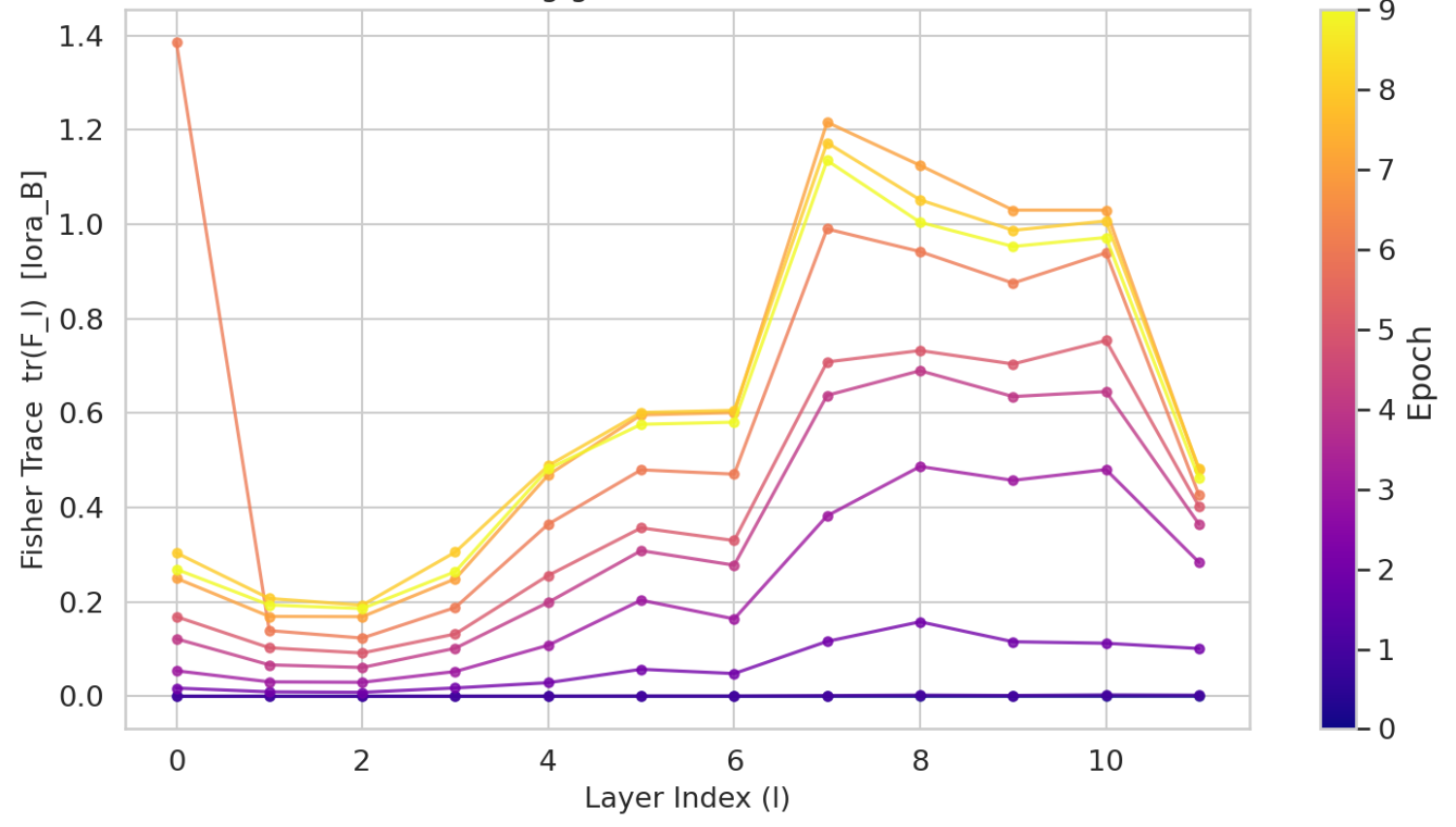}
        \caption{}
    \end{subfigure}
    
    \caption{Fisher trace vs layer index for Lora A (a) and Lora B (b) on RTE dataset.}
    \label{fig2}
\end{figure}

\begin{figure}[t]
    \centering
    \includegraphics[width=0.48\textwidth]{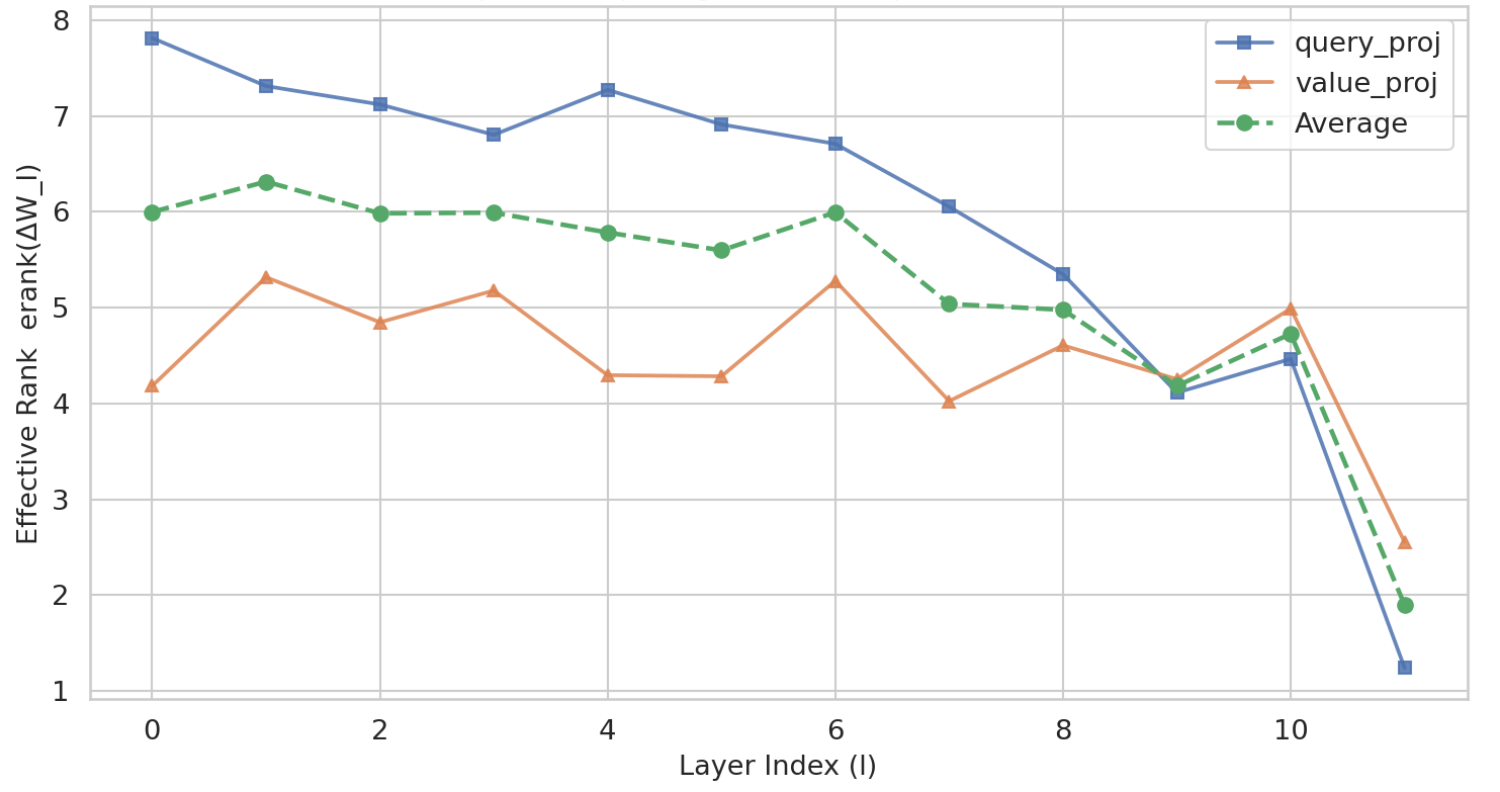}
    \caption{Effective rank vs layer index on RTE dataset.}
    \label{fig3}
\end{figure}

\paragraph{Contributions.}
\begin{itemize}
  \item We propose \textbf{LAARA}, a search-free framework for
        layer-adaptive LoRA rank allocation guided by lightweight
        Fisher information estimates.
  \item We introduce projection-wise normalization, logarithmic
        compression, and a blended adapter formulation for stable
        layer-wise importance estimation, alongside a vote-to-change
        dampening mechanism that prevents rank oscillation during
        training.
  \item Experiments on GLUE and mathematical reasoning demonstrate
        that LAARA consistently outperforms fixed-rank and
        adaptive-rank baselines at the same parameter budget.
\end{itemize}}

\section{Theoretical Analysis}
\label{sec:theory}

\textls[-20]{The empirical patterns discussed above suggest a deeper
mathematical structure governing how parameter importance shapes the
geometry of low-rank adaptation. We formalize this by addressing two
questions: \textbf{(RQ1)}~\emph{Is uniform rank allocation provably
suboptimal across transformer layers?} \textbf{(RQ2)}~\emph{Does the
required rank follow a structured pattern with layer depth?} We answer both affirmatively under mild conditions on the Fisher geometry of the pretrained model. Layers with concentrated Fisher spectra can be adapted with compact low-rank updates, while layers with broader spectra require higher-rank adapters to capture diverse optimization directions, making the effective rank of the layer-wise Fisher matrix a natural, principled estimate of the adapter rank required at each layer.}

\paragraph{Connecting theory to empirical observations.}
Observations~1 and~2 together imply that transformer layers differ
systematically in required adaptation capacity , formalized in
Theorem~1, which proves that uniform rank is suboptimal under
inter-layer Fisher heterogeneity. Observation~3 further shows that
this variation follows a structured depth-wise pattern, formalized in
Theorem~\ref{thm:rq2}, which establishes that gradient norms as Fisher
proxies induce a monotonic rank schedule across layers. Together, both
theorems provide a principled justification for layer-adaptive rank
allocation and directly motivate LAARA
(Sections~\ref{sec:rq1},~\ref{sec:rq2},~\ref{sec:rq1c}).

\subsection{Setup and Notation}
\label{sec:setup}

Consider a transformer with $L$ layers indexed by $\ell \in \{1,\dots,L\}$,
where larger $\ell$ denotes layers closer to the task loss. LoRA
parameterized the update at layer $\ell$ as $\Delta W_\ell = B_\ell A_\ell$,
where $B_\ell \in \mathbb{R}^{d_{\text{out}} \times r_\ell}$,
$A_\ell \in \mathbb{R}^{r_\ell \times d_{\text{in}}}$, and $r_\ell$ is
the adapter rank. Standard LoRA fixes $r_\ell = r$ for all $\ell$;
LAARA relaxes this constraint. Let $\theta_\ell = \{A_\ell, B_\ell\}$
denote the trainable parameters at layer $\ell$ and $\mathcal{L}(\theta)$
the downstream loss. The layer-wise FIM is:
\begin{equation}
\mathbf{F}_\ell =
\mathbb{E}_{(x,y)\sim\mathcal{D}}\!\left[
\nabla_{\theta_\ell}\log p(y|x;\theta)\,
\nabla_{\theta_\ell}\log p(y|x;\theta)^\top
\right].
\end{equation}
To quantify the intrinsic dimensionality of $\mathbf{F}_\ell$, we use
the effective rank:
\begin{equation}
\mathrm{erank}(\mathbf{F}_\ell)
= \exp\!\left(-\sum_{i=1}^{p}\bar{\lambda}_i\log\bar{\lambda}_i\right),
\quad
\bar{\lambda}_i = \frac{\lambda_i}{\sum_j \lambda_j},
\end{equation}
which measures the Shannon entropy of the normalized eigenvalue
distribution: a low value indicates a concentrated spectrum adaptable
with a compact update; a high value demands larger adapter capacity.
We denote by $r_\ell^{*}$ the \emph{optimal rank} of layer $\ell$:
the minimum rank sufficient to reduce approximation error below a
prescribed tolerance $\varepsilon > 0$ (Definition~\ref{def:optimal-rank}).

\subsection{RQ1: Uniform Rank Allocation is Suboptimal}
\label{sec:rq1}

We first formalize the conditions under which uniform rank allocation 
provably fails. Two mild assumptions on the Fisher geometry of the 
pretrained model suffice.

\subsubsection{Assumptions}

\begin{assumption}[Positive-definite Fisher]
\label{ass:pd12}
For every layer $\ell \in \{1,\ldots,L\}$, the Fisher Information Matrix 
$\bm{F}_{\ell}$ is positive definite, i.e.\ all eigenvalues satisfy 
$\lambda_{i}(\bm{F}_{\ell}) > 0$.
\end{assumption}

\textbf{Remark 1.}
Assumption~\ref{ass:pd12} holds generically when the model is identifiable 
and the training distribution has full support over the input space — 
conditions satisfied in practice for large language models finetuned on sufficiently rich datasets~\citep{6790500,pmlr-v89-karakida19a}.

\begin{assumption}[Heterogeneous Fisher spectra]
\label{ass:hetero123}
There exist two layers $\ell, \ell' \in \{1,\ldots,L\}$, $\ell \ne \ell'$, 
whose Fisher matrices have different effective ranks:
\[
  \mathrm{erank}(\bm{F}_{\ell}) \;\ne\; \mathrm{erank}(\bm{F}_{\ell'}).
\]
\end{assumption}

\textbf{Remark 2.}
Assumption~\ref{ass:hetero123} is well-supported both theoretically and empirically. \citet{zhang2023adalora} showed via singular value analysis that importance scores of incremental weight matrices vary substantially across transformer layers. \citet{Zhao_2025_CVPR} demonstrate using Fisher information that different layers exhibit markedly different parameter 
sensitivities. \citet{pmlr-v89-karakida19a} proved theoretically that the 
eigenvalue distribution of the FIM in deep networks is highly non-uniform 
across layers which is consistent with our empirical finding in Fig.~\ref{fig2}, where the Fisher trace of Layer~11 exceeds 
that of Layer~0 by nearly two orders of magnitude.

\subsubsection{Optimal Rank and Its Fisher Lower Bound}

\begin{definition}[Optimal LoRA rank]
\label{def:optimal-rank}
Fix a tolerance $\varepsilon > 0$. The \emph{optimal rank} $r_{\ell}^{*}$ 
for layer $\ell$ is the smallest rank for which a LoRA adapter can match 
the full-rank update to within $\varepsilon$:
{\scriptsize
\[
  r_{\ell}^{*} \;=\; \min\left\{ r \in \mathbb{Z}_{>0} \;\Big|\;
  \min_{\bm{B}_{\ell},\bm{A}_{\ell}}
  \mathbb{E}\!\left[\mathcal{L}(\bm{W}_{\ell} + \bm{B}_{\ell}\bm{A}_{\ell})
  \right] - \mathcal{L}^{*} \;\le\; \varepsilon \right\},
\]
}
where $\mathcal{L}^{*}$ is the loss achieved by the optimal full-rank 
update at layer $\ell$, with all other layers held fixed.
\end{definition}

The following two lemmas connect the optimal rank to the Fisher geometry of each layer and establish that heterogeneous Fisher spectra necessarily imply heterogeneous optimal ranks.

\begin{lemma}[FIM effective rank lower-bounds optimal rank]
\label{lem:fim-rank123}
Under Assumption~\ref{ass:pd12}, for any tolerance $\varepsilon > 0$ 
there exists a constant $c(\varepsilon) > 0$, depending only on 
$\varepsilon$ and $\mathcal{L}$, such that $r_{\ell}^{*} \;\ge\; c(\varepsilon) \cdot \mathrm{erank}(\bm{F}_{\ell}).$
\end{lemma}


\textbf{Remark 1.}
Lemma~\ref{lem:fim-rank123} establishes that layers with higher
effective Fisher rank require higher-rank LoRA adapters to achieve the
same approximation tolerance, directly formalizing the empirical
observation that later transformer layers which accumulate
substantially larger Fisher traces (Figure~\ref{fig3}), cannot be
adequately adapted with a uniform rank.

\begin{lemma}[Effective rank varies across layers]
\label{lem:hetero-rank}
Under Assumptions~\ref{ass:pd12} and~\ref{ass:hetero123}, the optimal 
ranks $r_{\ell}^{*}$ are not identical across all layers:
\[
  \exists\; \ell, \ell' : \quad r_{\ell}^{*} \;\ne\; r_{\ell'}^{*}.
\]
\end{lemma}

\subsubsection{Main Theorem for RQ1}

\begin{theorem}[Uniform rank allocation is suboptimal]
\label{thm:rq111}
\end{theorem}
Under Assumptions~\ref{ass:pd12} and~\ref{ass:hetero123}, for any fixed 
rank $r \in \mathbb{Z}_{>0}$ and any tolerance $\varepsilon > 0$, the 
uniform rank assignment $r_{\ell} = r$ for all $\ell$ is suboptimal in 
the following sense: there exists at least one layer $\ell^{+}$ for which 
the rank-$r$ adapter fails to achieve tolerance $\varepsilon$, we define this as follows:
\[
  \min_{\bm{B}_{\ell^{+}},\bm{A}_{\ell^{+}}}
  \mathbb{E}\!\left[
    \mathcal{L}(\bm{W}_{\ell^{+}} + \bm{B}_{\ell^{+}}\bm{A}_{\ell^{+}})
  \right] - \mathcal{L}^{*} \;>\; \varepsilon,
\]
or at least one layer $\ell^{-}$ for which rank $r$ strictly exceeds the 
minimum rank needed to achieve tolerance $\varepsilon$:
\[
  r > r_{\ell^{-}}^{*}.
\]
Moreover, the total excess loss under uniform rank satisfies:
{\small
\[
\begin{aligned}
\sum_{\ell=1}^{L}
\left(
\mathcal{L}(\bm{\phi}_{r}^{*,\ell}) - \mathcal{L}(\bm{\phi}^{*,\ell})
\right)
\;\ge\;& \\
\frac{c(\varepsilon)}{2}
\sum_{\ell=1}^{L}
\max\!\left(0,\;\mathrm{erank}(\bm{F}_{\ell}) - r\right)^{2}
\lambda_{\min}(\bm{F}_{\ell})\,\|\bm{\phi}^{*,\ell}\|^{2}
\end{aligned}
\]
}
which is strictly positive whenever $r < \max_{\ell}\,
\mathrm{erank}(\bm{F}_{\ell})$.

\textbf{Remark 1.}
Theorem~\ref{thm:rq111} answers \textbf{RQ1} definitively. The excess loss 
bound grows with the gap between the uniform rank $r$ and the layer-wise 
effective Fisher rank layers that are under-ranked contribute a penalty 
proportional to $\lambda_{\min}(\bm{F}_{\ell})$, the smallest eigenvalue 
of their Fisher matrix. Crucially, no fixed choice of $r$ simultaneously 
minimizes this bound across all layers: setting $r$ high enough to cover 
the most informative layers wastes capacity on less informative ones, 
while setting $r$ low wastes the budget of informative layers.

\subsection{RQ2: Gradient Norms Induce a Structured Rank
Pattern Across Layers}
\label{sec:rq2}

We now show that the expected gradient norm at layer $l$
is monotonically non-increasing as the layer index moves
away from the loss, establishing that layer depth induces
a structured pattern in the quantity that governs rank need.

\subsubsection{Assumptions}

\begin{assumption}[Bounded, Lipschitz-continuous activation]
\label{ass:activation12}
Each transformer sub-layer applies an activation function
$\phi : \mathbb{R} \to \mathbb{R}$ that is
$\beta$-Lipschitz, i.e.\ $|\phi'(z)| \le \beta$ for all
$z \in \mathbb{R}$, with $0 < \beta \le 1$.
\end{assumption}


\begin{assumption}[Bounded weight matrices]
\label{ass:bounded-weights12}
For each layer $l$, the spectral norm of the weight
matrix satisfies $\|\bm{W}_{l}\|_{2} \le \gamma$ for
some finite constant $\gamma > 0$.
\end{assumption}




\subsubsection{Lemmas}

\begin{lemma}[Backpropagation recurrence]
\label{lem:backprop}
Let $\bm{g}_{l} = \nabla_{\bm{h}_{l}}\mathcal{L}$ denote
the gradient of the loss with respect to the hidden
state $\bm{h}_{l} \in \mathbb{R}^{d}$ at layer $l$.
Under Assumptions~\ref{ass:activation12}
and~\ref{ass:bounded-weights12}, the backward pass satisfies
the recurrence:
\[
  \bm{g}_{l-1}
  \;=\;
  \bm{J}_{l}^{\top}\bm{g}_{l},
  \quad
  l = L, L-1, \ldots, 2,
\]
where $\bm{J}_{l} = \frac{\partial \bm{h}_{l}}
{\partial \bm{h}_{l-1}} \in \mathbb{R}^{d \times d}$
is the Jacobian of layer $l$'s mapping, and
$\|\bm{J}_{l}\|_{2} \le \beta\gamma$.
\end{lemma}

\begin{lemma}[Gradient norm is non-increasing toward
earlier layers]
\label{lem:grad-monotone}
Under Assumptions~\ref{ass:activation12},
\ref{ass:bounded-weights12},
if $\beta\gamma \le 1$, then for all $l \in
\{2, \ldots, L\}$:
\begin{equation}
\label{eq:grad-monotone}
  \|\bm{g}_{l-1}\|_{2}
  \;\le\;
  \|\bm{g}_{l}\|_{2}.
\end{equation}
Consequently, the sequence $(\|\bm{g}_{l}\|_{2})_{l=1}^{L}$
is monotonically non-decreasing from layer $1$ to layer $L$.
\end{lemma}

\begin{lemma}[Gradient norm of weight parameters is
non-increasing toward earlier layers]
\label{lem:weight-grad-monotone}
Under the same assumptions as Lemma~\ref{lem:grad-monotone},
the expected squared Frobenius norm of the gradient
with respect to the weight matrix $\bm{W}_{l}$ satisfies:
\[
  \mathbb{E}\!\left[
    \|\nabla_{\bm{W}_{l-1}}\mathcal{L}\|_{F}^{2}
  \right]
  \;\le\;
  (\beta\gamma)^{2}
  \mathbb{E}\!\left[
    \|\nabla_{\bm{W}_{l}}\mathcal{L}\|_{F}^{2}
  \right].
\]
\end{lemma}
\subsubsection{Main Theorem for RQ2}

\begin{theorem}[Gradient norms induce a structured,
monotonic rank pattern]
\label{thm:rq2}
\end{theorem}
Under Assumptions~\ref{ass:activation12},
\ref{ass:bounded-weights12},
with $\beta\gamma \le 1$, the expected gradient
Frobenius norm satisfies
{\small
\begin{equation}
\label{eq:gradient-chain}
  \mathbb{E}\!\left[
    \|\nabla_{\bm{W}_{l}}\mathcal{L}\|_{F}^{2}
  \right]
  \;\le\;
  (\beta\gamma)^{2(L-l)}\,
  \mathbb{E}\!\left[
    \|\nabla_{\bm{W}_{L}}\mathcal{L}\|_{F}^{2}
  \right]
  \quad \forall\, l \le L.
\end{equation}
}
Since the expected gradient Frobenius norm is a
consistent estimator of $\mathrm{tr}(\bm{F}_{l})$,
the layer-wise Fisher trace satisfies:
\begin{equation}
\label{eq:fisher-trace-chain}
  \mathrm{tr}(\bm{F}_{l})
  \;\le\;
  (\beta\gamma)^{2(L-l)}\,\mathrm{tr}(\bm{F}_{L})
  \quad \forall\, l \le L.
\end{equation}
Consequently, from lemma~\ref{lem:fim-rank123} the
optimal rank $r_{l}^{*}$ satisfies:
\begin{equation}
\label{eq:rank-pattern}
  r_{l}^{*}
  \;\le\;
  r_{l+1}^{*}
  \quad \forall\, l \in \{1,\ldots,L-1\},
\end{equation}
i.e.\ the optimal rank is \emph{non-decreasing} from
the first layer to the last, forming a structured
monotonic profile with depth. The complete proofs of the lemmas and theorems are provided in Appendix Section~\ref{sec:rq1c}.

\section{Proposed Method}
\label{sec:method}
Building on the empirical and theoretical analysis above, we present
LAARA, which allocates layer-wise adapter capacity from Fisher
information estimates computed during training across four components: dynamic LoRA parameterization, Fisher information estimation, adaptive rank allocation, and a vote-based dampening mechanism.

\subsection{Dynamic LoRA Parameterization}
\label{sec:dynamic_lora}

LoRA \cite{hu2022lora} adapts a pretrained weight matrix
$W_0 \in \mathbb{R}^{m \times n}$ by introducing a low-rank update
$\Delta W = BA$, where
$B \in \mathbb{R}^{d_{\text{out}} \times r}$ and
$A \in \mathbb{R}^{r \times d_{\text{in}}}$ with
$r \ll \min(m,n)$. Unlike standard LoRA, which employs a fixed uniform rank across all layers, our method dynamically adjusts the rank during training. The adapted weight matrix is represented as:
\[
W = W_0 + \frac{\alpha}{r}\Delta W,
\qquad
\Delta W = BA,
\]
where $\alpha$ is a scaling hyperparameter used to maintain stable output magnitudes across different rank values. For each targeted linear projection
$W \in \mathbb{R}^{d_{\text{out}} \times d_{\text{in}}}$,
LoRA introduce trainable low-rank adapter matrices
$A \in \mathbb{R}^{r \times d_{\text{in}}}$
and
$B \in \mathbb{R}^{d_{\text{out}} \times r}$.
The forward pass is defined as:
\begin{equation}
  h = Wx + \frac{\alpha}{r}\,BAx,
  \label{eq:lora_forward}
\end{equation}
where the pretrained weight matrix $W$ remains frozen throughout training.
Following the standard LoRA initialization strategy, $A$ is initialized from a normal distribution while $B$ is initialized to zero, ensuring that the adapter initially contributes no perturbation to the pretrained model.




\paragraph{Rank adaptation.}
When layer $\ell$ is updated from rank $r$ to $r'$, adapter matrices
are resized in-place. For rank increases ($r' > r$), the additional
$r'-r$ rows of $A$ are randomly initialised and corresponding columns
of $B$ set to zero, preserving learned directions while adding
capacity. For rank decreases ($r' < r$), components are pruned by
importance score:
\begin{equation}
  \phi_i = \|B[:,i]\|_2 \cdot \|A[i,:]\|_2, \quad i = 1,\ldots,r,
  \label{eq:importance}
\end{equation}
retaining only the top-$r'$ components, consistent with prior
low-rank pruning~\citep{hu2022lora,zhang2023adalora}. To determine
how capacity should be redistributed across layers, we estimate
layer sensitivity via empirical Fisher information derived from
adapter gradients.

\subsection{Fisher Information Estimation}
\label{sec:fisher}

The diagonal Fisher information for a parameter $\theta$ under loss $\mathcal{L}$ is defined as:
\begin{equation}
  \mathcal{F}(\theta) =
  \mathbb{E}\!\left[
  \left(
  \frac{\partial \mathcal{L}}{\partial \theta}
  \right)^2
  \right].
  \label{eq:fisher_diag}
\end{equation}

We approximate $\mathcal{F}$ using the empirical trace of squared gradients over a mini-batch:
\begin{equation}
  F_l^{(t)} =
  \sum_i \left(g_i^{(t)}\right)^2,
  \label{eq:fisher_trace}
\end{equation}
where $g_i^{(t)}$ denotes the gradient of the loss with respect to the $i^{\text{th}}$ LoRA parameter in layer $l$ at optimization step $t$. We maintain separate Fisher estimates for the $A$ and $B$ adapter matrices, denoted by $F_{l,A}^{(t)}$ and $F_{l,B}^{(t)}$, respectively. To reduce gradient variance and stabilize importance estimation, both quantities are tracked using an exponential moving average (EMA):
\begin{equation}
  \hat{F}_{l,\cdot}^{(t)}
  =
  \beta \hat{F}_{l,\cdot}^{(t-1)}
  +
  (1-\beta)F_{l,\cdot}^{(t)},
  \label{eq:ema}
\end{equation}
where $\beta \in (0,1)$ denotes the EMA decay factor.
Bias correction is further applied to compensate for initialization effects during early optimization steps:
\begin{equation}
  \tilde{F}_{l,\cdot}^{(t)}
  =
  \frac{\hat{F}_{l,\cdot}^{(t)}}{1-\beta^t}.
  \label{eq:bias_correction}
\end{equation}

The resulting Fisher estimates provide a measure of layer-wise adaptation sensitivity. However, raw Fisher values often exhibit substantial scale variations across projection types, making them unsuitable for direct rank assignment. We therefore introduce a normalized and compressed scoring mechanism to derive stable rank allocation signals.

\subsection{Rank Allocation}
\label{sec:rank_alloc}

\paragraph{Per-projection normalization.}
Raw Fisher traces can differ significantly across projection types (e.g., query versus value projections), causing certain projections to dominate the allocation process. To mitigate this issue, we normalize Fisher scores independently within each projection type
$p \in \{\texttt{query\_proj}, \texttt{value\_proj}, \ldots\}$:
\begin{equation}
  s_{l,p}
  =
  \frac{
  \tilde{F}_{l,p}
  -
  \min_{l'} \tilde{F}_{l',p}
  }{
  \max_{l'} \tilde{F}_{l',p}
  -
  \min_{l'} \tilde{F}_{l',p}
  +
  \epsilon
  },
  \label{eq:norm}
\end{equation}
yielding normalized per-projection scores
$s_{l,p} \in [0,1]$.


\paragraph{Log compression.}
To reduce the influence of extreme outliers and encourage smoother rank distributions, we apply logarithmic compression:
\begin{equation}
  \hat{s}_{l,p}
  =
  \frac{
  \log(1+\gamma s_{l,p})
  }{
  \log(1+\gamma)
  },
  \label{eq:log_compress}
\end{equation}
where $\gamma > 0$ controls the compression strength. This transformation preserves the relative ordering of layers while attenuating disproportionately large Fisher responses.


\paragraph{Blended signal.}
To obtain a unified layer-wise importance estimate, we combine the compressed scores derived from the $A$ and $B$ adapters using weighted aggregation:
\begin{equation}
  \hat{s}_l
  =
  \alpha_b \hat{s}_{l,A}
  +
  (1-\alpha_b)\hat{s}_{l,B},
  \label{eq:blend}
\end{equation}
where $\alpha_b \in [0,1]$ controls the relative contribution of the two adapter components. This blended formulation captures complementary adaptation dynamics from both low-rank projection matrices, resulting in a more stable and representative importance signal for rank allocation. Setting $\alpha_b = 1$ yields a signal derived solely from the $A$ adapter, while $\alpha_b = 0.5$ assigns equal importance to both matrices. Each projection type is normalized independently before blending to avoid scale inconsistencies.



\paragraph{Integer rank assignment.}
The blended importance score is subsequently transformed into a discrete layer-wise LoRA rank using a bounded linear mapping:
{\small
\begin{equation}
  r_l
  =
  \operatorname{clip}
  \left(
  \operatorname{round}
  \left(
  r_{\min}
  +
  (r_{\max}-r_{\min})\hat{s}_l
  \right),
  r_{\min},
  r_{\max}
  \right),
  \label{eq:rank_assign}
\end{equation}
}
where $r_{\min}$ and $r_{\max}$ denote the minimum and maximum allowable adapter ranks, respectively.

This formulation allocates larger ranks to Fisher-sensitive layers while compressing under-utilized layers, improving parameter efficiency without degrading downstream performance. Although the adaptive allocation strategy dynamically redistributes representational capacity during training, noisy gradient estimates may still induce unstable rank fluctuations. To address this issue, we introduce a vote-based dampening mechanism that stabilizes rank transitions over time.


\subsection{Vote-to-Change Dampening}
\label{sec:dampening}

Frequent rank changes can destabilize optimization because each update partially disrupts optimizer momentum associated with the affected parameters. To suppress rapid oscillations, we introduce a \emph{vote-to-change} mechanism in which a rank update for layer $l$ is committed only when the same target rank is proposed consistently for $\tau$ consecutive update steps:
\begin{equation}
  r_l^{\star} =
  \begin{cases}
    r_l, & \text{if the last } \tau \text{ proposals equal } r_l, \\
    r_l^{\text{current}}, & \text{otherwise.}
  \end{cases}
  \label{eq:dampening}
\end{equation}

Once a rank update is accepted, the corresponding vote buffer is reset, requiring the layer to accumulate another sequence of $\tau$ consistent proposals before any subsequent modification can occur. This hysteresis-inspired mechanism suppresses cyclic rank oscillations caused by greedy updates and leads to more stable adaptive rank convergence. The algorithm of the LAARA framework is described in Algorithm~\ref{alg:laara}.

\begin{figure}[t]
    \centering
    \includegraphics[width=0.48\textwidth]{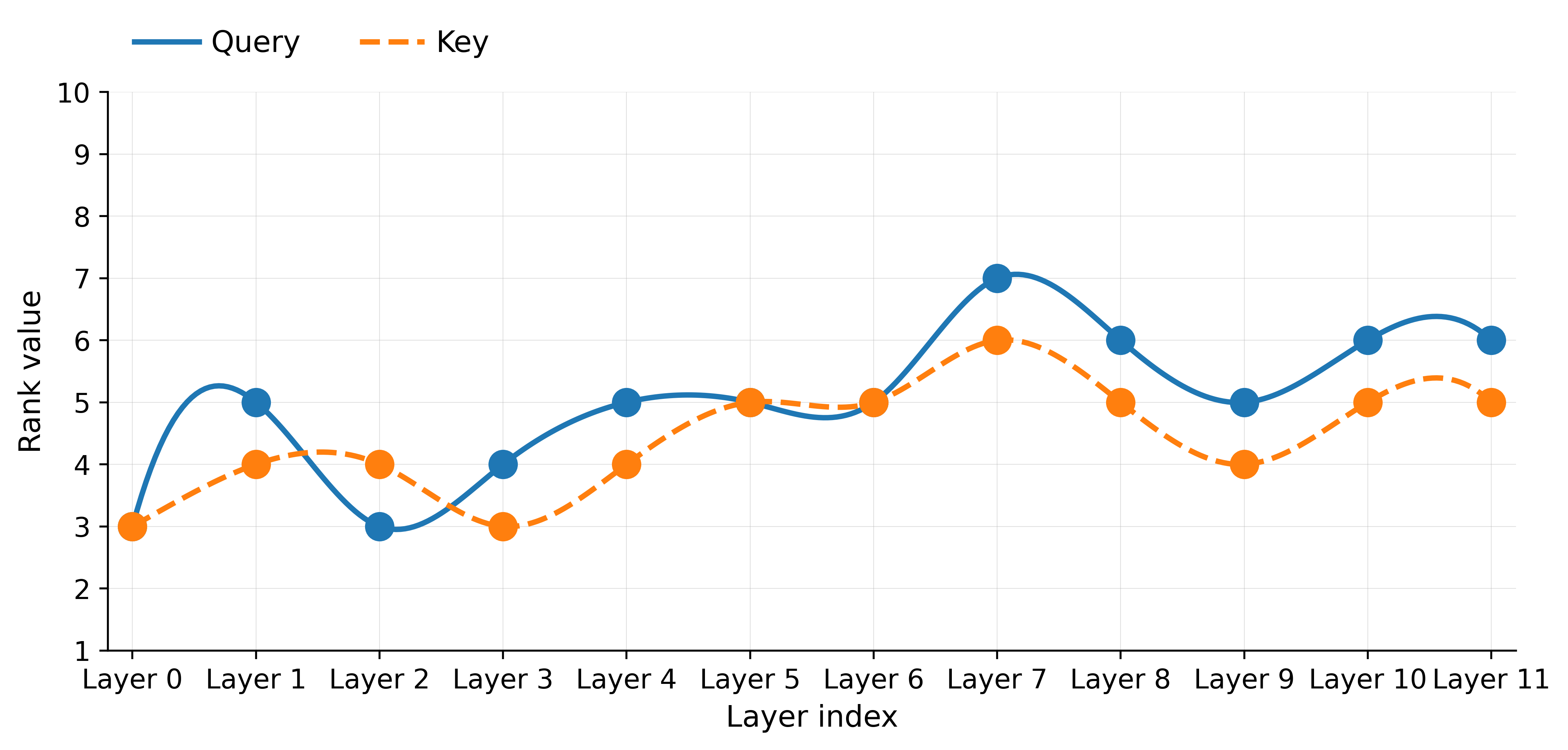}
    \caption{Final rank value vs layer index on STS-B dataset.}
    \label{fig3xyz}
\end{figure}

\begin{table*}[t]
\centering
\small
\setlength{\tabcolsep}{4pt}

\begin{tabular}{lccccccccc}
\toprule

\textbf{Method} &
\textbf{\# Params} &
\textbf{MNLI} &
\textbf{SST-2} &
\textbf{CoLA} &
\textbf{QQP} &
\textbf{QNLI} &
\textbf{RTE} &
\textbf{MRPC} &
\textbf{STS-B} \\

&
&
\textbf{m/mm} &
\textbf{Acc} &
\textbf{MCC} &
\textbf{F1} &
\textbf{Acc} &
\textbf{Acc} &
\textbf{F1} &
\textbf{Pearson} \\

\midrule

Full FT
& 185,310,724
& \textbf{91.50/91.60}
& 96.10
& 62.23
& 89.50
& 93.80
& 76.89
& 89.90
& 91.90 \\

BitFit
& 619,778
& 89.10/88.90
& 94.04
& 66.49
& 83.86
& 89.35
& 78.70
& 90.94
& 91.14 \\

LoRA
& 887,042
& 90.55/90.32
& 94.83
& 68.00
& 85.40
& 93.30
& 85.20
& 92.55
& 91.64 \\

AdaLoRA
& 1,330,563
& 90.88/90.71
& 95.64
& 66.79
& 88.71
& 91.48
& 81.59
& 93.24
& 91.86 \\

DyLoRA
& 887,042 
& 89.99/89.72
& 95.76
& 65.67
& 89.40
& 92.55
& 85.92
& 89.22
& 90.87 \\

\midrule

\textbf{LAARA (Ours)}
& \textbf{851,522}
& 91.42/91.18
& \textbf{96.02}
& \textbf{69.18}
& \textbf{89.86}
& \textbf{94.11}
& \textbf{87.73}
& \textbf{93.45}
& \textbf{91.65} \\

\bottomrule
\end{tabular}

\caption{Comparison of parameter-efficient fine-tuning methods on GLUE benchmarks.}
\label{tab:glue_results}

\end{table*}

\begin{table}[t]
\centering
\small
\setlength{\tabcolsep}{6pt}

\begin{tabular}{lccccc}
\toprule

\textbf{Method} &
\textbf{\# Params} &
\textbf{CoLA} &
\textbf{STS-B} &
\textbf{RTE} \\

&
&
\textbf{MCC} &
\textbf{Pearson} &
\textbf{Acc} \\

\midrule

BitFit
& 619,778
& 61.65
& 89.47
& 73.29 \\

LoRA
& 1,034,498
& 68.00
& 91.62
& 84.48 \\

DyLoRA
& 1,034,498
& 63.20
& 91.19
& 87.73 \\

AdaLoRA
& 1,256,114
& 69.20
& 91.87
& 80.51 \\

\midrule

\textbf{LAARA (Ours)}
& \textbf{987,042}
& \textbf{68.88}
& \textbf{91.91}
& \textbf{88.09} \\

\bottomrule
\end{tabular}

\caption{Performance comparison of parameter-efficient fine-tuning methods using query, key, value, and output projection matrices ($W_q$, $W_k$, $W_v$, $W_o$) as target modules.}
\label{tab:target_modules_results}

\end{table}

\section{Experiments}
\label{sec:experiments}

\subsection{Experimental Setup}
\label{sec:setup}

\paragraph{Model and datasets.}
\textls[-20]{We evaluated the proposed LAARA approach and the baselines on DeBERTa-v3-base~\citep{he2023debertav3improvingdebertausing} and Llama-3.2-3B~\citep{grattafiori2024llama3herdmodels} and fine-tuned across all eight tasks of the GLUE benchmark~\citep{wang-etal-2018-glue}: MNLI, SST-2, CoLA, QQP, QNLI, RTE, MRPC, and STS-B, and on Mathintruct Dataset~\citep{yue2023mammothbuildingmathgeneralist}. 
The implementation details and baselines are provided in the Section~\ref{sec:impl_details} of the Appendix.}

\begin{table}[t]
\centering
\small
\begin{tabular}{lcc}
\toprule
\textbf{Method} & \textbf{\# Params} & \textbf{Accuracy (\%)} \\
\midrule
BitFit          & 98,304    & 19.40 \\
LoRA            & 4,587,520 & 24.50 \\
DyLoRA          & 4,587,520 & 23.10 \\
AdaLoRA         & 5,160,960 & 23.80 \\
\midrule
\textbf{LAARA (Ours)} & \textbf{3,808,512} & \textbf{27.20} \\
\bottomrule
\end{tabular}
\caption{Performance comparison on MathInstruct dataset using
Llama-3.2-3B.}
\label{tab:mathinstruct}
\end{table}

\begin{table}[t]
\centering
\small
\setlength{\tabcolsep}{6pt}

\begin{tabular}{lcccc}
\toprule

\textbf{Method} &
\textbf{\# Params} &
\textbf{CoLA} &
\textbf{STS-B} &
\textbf{RTE} \\

&
&
\textbf{MCC} &
\textbf{Pearson} &
\textbf{Acc} \\

\midrule

BitFit
& 619,778
& 61.65
& 89.47
& 73.29 \\

LoRA
& 8,275,984
& 70.60
& 91.67
& 79.78 \\

DyLoRA
& 8,275,984
& 67.92
& 91.42
& 82.31 \\

AdaLoRA
& 8,942,336
& 70.11
& 91.88
& 78.94 \\

\midrule

\textbf{LAARA (Ours)}
& \textbf{6,104,512}
& \textbf{71.08}
& \textbf{92.01}
& \textbf{84.66} \\

\bottomrule
\end{tabular}













\caption{Performance comparison with rank $r=64$ using query, key, value, and output projection matrices ($W_q$, $W_k$, $W_v$, $W_o$) as target modules.}
\label{tab:r64_results}

\end{table}



\subsection{Results and Analysis}
\label{sec:main_results}

\textls[-20]{LAARA achieves the best performance across all GLUE tasks
(Table~\ref{tab:glue_results}), with the strongest gains on reasoning-intensive
tasks. On RTE, LAARA reaches 87.73\% accuracy, outperforming LoRA by
+2.53\% and AdaLoRA by +6.14\% at identical parameter budget. On MNLI,
QQP, and QNLI, LAARA consistently leads all PEFT baselines, and notably
surpasses full fine-tuning by using a smaller parameters~\citep{hu2022lora, zhang2023adalora}. AdaLoRA underperforms LAARA on five tasks despite its SVD-based rank pruning, most notably on RTE ($-$6.14\%) and QNLI ($-$2.63\%), which we
attribute to the sensitivity of singular value importance scores to
initialization and regularization on small datasets. DyLoRA similarly
trails LAARA on most tasks, confirming that rank diversity alone without
principled layer-wise allocation is insufficient. By contrast, LAARA
derives rank assignments directly from Fisher information estimates,
yielding a statistically grounded, search-free strategy that is robust
across task sizes and domains. 
As shown in Fig.~\ref{fig3xyz} on the STS-B dataset, the final layer-wise ranks learned after training strongly validate our motivation. The proposed Fisher-guided LAARA framework consistently assigns lower ranks to early layers and progressively higher ranks to deeper layers, exhibiting a clear monotonic adaptation pattern across the network. We have discussed the result for Table~\ref{tab:mathinstruct} in the section~\ref{sec:mathinstruct} in the Appendix.}

\subsection{Parameter Efficiency}
\label{sec:param_efficiency}
LAARA consistently reduces trainable parameters while improving
performance across all settings. With $W_q, W_k, W_v, W_o$ as
target modules (Table~\ref{tab:target_modules_results}), LAARA
uses 987,042 parameters, a 4.6\% fewer than LoRA/DyLoRA
(1,034,498) and 21.4\% fewer than AdaLoRA (1,256,114), while
achieving the best results on RTE (88.09\%) and STS-B (91.91\%).
At $r{=}64$ (Table~\ref{tab:r64_results}), the gap widens:
LAARA uses 6,104,512 parameters, a 26.2\% fewer than LoRA/DyLoRA
(8,275,984) and 31.7\% fewer than AdaLoRA (8,942,336), while
outperforming all baselines on every task. On the generative
reasoning setting (Table~\ref{tab:mathinstruct}), LAARA uses
3,808,512 parameters, a 17.0\% fewer than LoRA/DyLoRA
(4,587,520) and 26.2\% fewer than AdaLoRA (5,160,960),
while achieving the highest accuracy (27.20\%). These reductions
follow directly from Fisher-guided allocation: layers with low
Fisher trace are assigned $r_{\min}$, reclaiming capacity that
uniform-rank methods waste on less informative layers, with
savings growing as the rank budget increases.

\section{Conclusion}
\label{sec:conclusion}
We presented LAARA, a parameter-efficient fine-tuning framework that
allocates layer-wise ranks according to Fisher information
estimates, without iterative search or auxiliary optimization. By
concentrating adapter capacity in layers with high loss-landscape
curvature and compressing it elsewhere, LAARA addresses the provable
suboptimality of uniform rank allocation. A combination of EMA-smoothed Fisher
estimation, projection-wise normalization, logarithmic compression,
and vote-based dampening ensures that rank transitions are both
principled and stable throughout training. Experimental results show that LAARA consistently matches or outperforms state of the art baselines at the same parameter budget, with particularly strong gains on reasoning-intensive tasks such as RTE.




\section{Limitations}
\label{sec:limitations}

While LAARA consistently outperforms existing PEFT baselines, the
diagonal Fisher approximation discards off-diagonal curvature and may
underestimate true layer importance in some settings. Key
hyperparameters including the rank update interval and
\texttt{Lora A}/\texttt{Lora B} blend coefficient are fixed across
tasks without extensive ablation. Future work will explore richer
Fisher approximations, adaptive scheduling, and broader
multi-architecture evaluations.

\bibliography{custom}

\appendix

\section{Appendix}
\label{sec:appendix}

\subsection{Related Work}
\label{sec:related_work}
The increasing size of LLMs has made direct fine-tuning costly and often impractical, driving the development of a wide range of PEFT methods that enhance task-specific performance by introducing a small set of trainable parameters while keeping most pre-trained parameters fixed. Approaches like adapters~\cite{houlsby2019parameter}, prefix-tuning~\cite{li2021prefix}, and prompt-tuning~\cite{lester2021power} exemplify this paradigm by integrating lightweight modules or augmenting input and hidden layers. Among PEFT approaches, LoRA~\cite{hu2022lora} established a highly influential paradigm by decomposing weight updates into a pair of low-rank matrices while freezing the original pretrained weights. LoRA demonstrated that only a small fraction of trainable parameters is sufficient to achieve performance comparable to full fine-tuning across a variety of NLP tasks, while also avoiding additional inference latency through weight merging after training. This simplicity, efficiency, and strong empirical performance made LoRA the foundation for a rapidly growing family of PEFT methods. 

Despite its effectiveness, LoRA employs a fixed rank across all transformer layers, assuming uniform adaptation requirements throughout the model. However, later studies showed that different layers exhibit heterogeneous adaptation behaviour, causing uniform rank allocation to either waste parameters or limit adaptation capacity in critical layers~\cite{lin2026tlorataskawarelowrank}. This limitation motivated a growing line of research on adaptive and dynamic rank allocation strategies for more efficient and effective PEFT. AdaLoRA~\cite{zhang2023adalora} addressed this issue by reformulating LoRA updates through singular value decomposition and dynamically pruning singular values according to importance scores during training. DyLoRA~\cite{valipour2023dylora} further relaxed the dependency on manual rank selection by jointly training nested adapters across multiple ranks, enabling inference-time rank selection without retraining. Subsequent works such as ARD-LoRA~\cite{11151211}, DR-LoRA~\cite{deng2026drloradynamicranklora}, and RA-LoRA~\cite{kim-etal-2024-ra} extended this direction by introducing layer-aware or training-aware rank allocation strategies, highlighting that the optimal adaptation capacity varies substantially across layers and tasks. In parallel, SoRA~\cite{ding2023sparse} incorporated sparsity-inducing gates into LoRA singular values to suppress redundant components, while GLoRA~\cite{chavan2023oneforallgeneralizedloraparameterefficient} generalized the low-rank parameterization itself to support broader adaptation structures. AutoPEFT~\cite{zhou-etal-2024-autopeft} proposed an automatic configuration search framework for PEFT that jointly optimizes module type, insertion layers, and parameter budget using multi-objective Bayesian optimization. The method discovers Pareto-optimal and transferable PEFT configurations that outperform manually designed PEFT strategies while achieving performance comparable to full fine-tuning. However, AutoPEFT introduces additional search-time computational overhead and depends on a large configuration search space, which may limit efficiency and scalability for more complex large-scale tasks. L1RA~\cite{singh-etal-2025-l1ra} proposes a dynamic rank assignment strategy for LoRA fine-tuning that uses L1 regularization to prune redundant ranks and redistribute them across adapters under a fixed rank budget. The method improves parameter utilization and achieves comparable or better performance than standard LoRA while maintaining low computational overhead. However, the approach introduces additional optimization complexity for dynamic rank allocation and its effectiveness is primarily validated on limited benchmark settings, leaving large-scale generalization insufficiently explored.

Another line of research explored automated and data-driven mechanisms for determining adapter capacity. AutoLoRA~\cite{zhang2024autoloraautomaticallytuningmatrix} formulated rank selection as a meta-learning problem capable of generalizing across downstream tasks, although at the cost of expensive bi-level optimization. Dynamic Rank Assignment~\cite{singh-etal-2025-l1ra} instead employed online gradient-based heuristics to update ranks during training without relying on SVD decomposition. FlexoRA~\cite{wei-etal-2025-flexora} introduced a continuous relaxation of rank allocation through reparameterization, while Bayesian LoRA Search~\cite{lin2026bayesianloraprobabilisticlowrankadaptation} utilized Bayesian optimization to reduce the cost of exhaustive hyperparameter search. Collectively, these methods demonstrate that adaptive rank allocation is critical for improving the parameter-efficiency and robustness of LoRA-style fine-tuning, but many approaches still rely on iterative search procedures, auxiliary optimization objectives, or costly training-time adaptation.

Recent works have attempted to guide adaptation using activation statistics or information-aware criteria. AROMA~\cite{sheng-etal-2025-aroma} incrementally constructed adapters from rank-one components according to a relevance criterion, enabling autonomous low-rank matrix adaptation during training. Related studies have also investigated weight decomposition and compositional adaptation, such as DoRA~\cite{mao2024doraenhancingparameterefficientfinetuning}, which separates magnitude and directional updates for improved adaptation stability, and LoraHub, which studies the composability of independently trained LoRA modules across tasks.  These works collectively suggest that adaptation quality depends not only on parameter count, but also on how effectively the adapter subspace aligns with the underlying geometry of the pretrained model.

The Fisher Information Matrix (FIM) has long served as a principled tool for characterizing parameter importance and the geometry of neural network optimization. Natural gradient methods utilize the FIM as a Riemannian metric to perform optimization on the statistical manifold~\cite{martens2015optimizing}, while Elastic Weight Consolidation (EWC)~\cite{kirkpatrick2017ewc} employs the diagonal Fisher to preserve important parameters during continual learning. Fisher-based pruning and compression methods have further shown that Fisher information provides reliable estimates of parameter sensitivity and contribution to model predictions. Despite its strong theoretical grounding, Fisher information has been only minimally explored in the context of PEFT and adaptive low-rank tuning. Existing LoRA variants largely rely on gradient magnitudes, activation statistics, sparsity regularization, or heuristic search strategies for rank determination. In contrast, our work directly utilizes Fisher information to estimate the adaptation sensitivity of each layer and allocate LoRA ranks accordingly, providing a statistically grounded and search-free framework for parameter-efficient fine-tuning. FLoE~\cite{wang2025floefisherbasedlayerselection} leverages Fisher information to identify task-critical layers for sparse MoE-based adapter deployment and employs Bayesian optimization to automatically determine optimal layer-wise LoRA ranks without exhaustive hyperparameter search. Experimental results demonstrate that FLoE achieves strong efficiency–accuracy trade-offs, particularly in resource-constrained settings requiring rapid model adaptation. FiLoRA~\cite{han2026filora} introduces a Fisher information-guided PEFT framework that uses K-FAC-based curvature estimation to identify sensitive low-rank adaptation directions and optimize only a compact core matrix, thereby reducing trainable parameters while preserving performance. However, the method incurs additional computational overhead for Fisher matrix estimation and relies on static curvature information computed from limited downstream samples, which may restrict scalability and adaptability across diverse tasks.


\subsection{Implementation Details}
\label{sec:impl_details}

Following standard GLUE evaluation protocol, we report accuracy for
MNLI (matched/mismatched), SST-2, QNLI, and RTE; Matthews Correlation
Coefficient (MCC) for CoLA; F1-score for QQP and MRPC; and Pearson
correlation for STS-B. All results are averaged over three runs with
different random seeds. For all methods, the learning rate is selected
from $\{$4e-4, 5e-4, 8e-4, 1e-3, 1.2e-3, 2.2e-3$\}$ and the best
performing value per task is reported. LoRA and DyLoRA are configured
with uniform rank $r{=}8$, scaling $\alpha{=}16$, batch size 64, and
maximum sequence length 128, with the number of training epochs tuned
per task; DyLoRA additionally performs rank search over $[2, 8]$
during training without retraining. AdaLoRA uses initial rank
$r_i{=}12$, target rank $r_f{=}8$, scaling $\alpha{=}16$, batch size
32, sequence length 128, with per-task training steps, warmup steps,
rank update intervals, and regularization coefficient $\gamma$
following~\citet{zhang2023adalora}. LAARA uses initial rank
$r_{\text{init}}{=}4$, minimum rank $r_{\min}{=}2$, maximum rank
$r_{\max}{=}8$, scaling $\alpha{=}16$, EMA decay $\beta{=}0.97$,
warmup steps $T_w{=}200$, rank update interval $\Delta{=}200$, blend
coefficient $\alpha_b{=}0.5$, log compression $\gamma{=}10$, and
dampening patience $\tau{=}2$, with adapters applied to $W_q$ and
$W_v$ for the primary evaluation and additionally to $W_k$ and $W_o$
for the extended evaluation in Table~\ref{tab:target_modules_results}.
For Mathinstruct dataset, we have used max length=512, r=16,$\alpha=32$, learning rate = 2e-4, with batch size of 4 and ran till 3 epochs. We have evaluated this on GSM8K dataset~\citep{cobbe2021gsm8k}.

\paragraph{Baselines.}
We compare LAARA against five baselines spanning a range of fine-tuning strategies:
(i) Full fine-tuning (FT), which updates all model 
parameters and serves as an upper-bound reference;
(ii) BitFit~\citep{ben-zaken-etal-2022-bitfit}, a lightweight 
method that updates only bias terms;
(iii) LoRA~\citep{hu2022lora}, the standard fixed 
uniform-rank baseline;
(iv) AdaLoRA~\citep{zhang2023adalora}, which 
dynamically reallocates the parameter budget by pruning singular values during training; and
(v) DyLoRA~\citep{valipour2023dylora}, which trains 
adapters simultaneously across a range of ranks.

\subsection{Generalization to Mathematical Reasoning}
\label{sec:mathinstruct}

Table~\ref{tab:mathinstruct} extends evaluation to MathInstruct dataset using
Llama-3.2-3B, where LAARA achieves 27.20\% accuracy, outperforming
LoRA by +2.70\% and AdaLoRA by +3.40\% while using 17.0\% fewer
parameters than LoRA and 26.2\% fewer than AdaLoRA. DyLoRA and AdaLoRA unperformed despite larger parameter budgets, consistent with
our GLUE findings, confirming that Fisher-guided rank allocation
generalizes beyond NLU classification to generative reasoning tasks.

\subsection{Discussion}
\label{sec:discussion}

\paragraph{Why Fisher-guided allocation helps.}
The consistent gains of LAARA over fixed-rank LoRA are explained
by the three empirical observations.
Fisher traces vary by nearly two orders of magnitude across layers
(Observation~2), meaning that a uniform rank simultaneously
over-allocates capacity to early layers and under-allocates it to
late layers.
By assigning ranks proportional to per-layer Fisher importance,
LAARA concentrates its parameter budget where the loss landscape
has the most curvature, precisely the layers that contribute most
to downstream task adaptation.

\paragraph{Stability of rank allocation.}
A practical concern with dynamic rank methods is training
instability caused by frequent rank changes.
LAARA addresses this through two mechanisms: the EMA-smoothed
Fisher estimates (Eq.~\ref{eq:ema}) reduce gradient noise, and
the vote-to-change dampening (Eq.~\ref{eq:dampening}) prevents rank
oscillation by requiring $\tau = 2$ consecutive consistent
proposals before committing a rank update.
Together these ensure that rank transitions are both principled
and stable, without introducing additional hyperparameters beyond
those already present in standard LoRA.

\begin{algorithm}[t]
\caption{\textit{LAARA}}
\label{alg:laara}
\begin{algorithmic}[1]

\Require{Pretrained model $\Theta$, LoRA adapters $\{A_{l,p}, B_{l,p}\}$, minimum rank $r_{\min}$, maximum rank $r_{\max}$, EMA decay $\beta$, compression factor $\gamma$, blend coefficient $\alpha_b$, vote threshold $\tau$, rank update interval $\Delta T$}
\Ensure{Updated LoRA parameters}

\State Initialize: $r_l \leftarrow r_{\max}$, $\hat{F}^{(0)}_{l,p} \leftarrow 0$, $v_l \leftarrow 0$ \; $\forall\, l, p$

\For{$t = 1, 2, \cdots, T$} \Comment{Training Loop}
    \State Compute task loss $\mathcal{L}^{(t)}$ and backpropagate
    \For{all layers $l$ and projection types $p$}
        \State $F^{(t)}_{l,p} \leftarrow \sum_i \!\left(\frac{\partial \mathcal{L}^{(t)}}{\partial \theta_{l,p,i}}\right)^{\!2}$ \Comment{Fisher}
        \State $\hat{F}^{(t)}_{l,p} \leftarrow \beta\, \hat{F}^{(t-1)}_{l,p} + (1-\beta)\, F^{(t)}_{l,p}$ \Comment{EMA}
        \State $\tilde{F}^{(t)}_{l,p} \leftarrow \hat{F}^{(t)}_{l,p} \big/ (1 - \beta^t)$ \Comment{Bias Correction}
    \EndFor
    \If{$\mathrm{MOD}(t,\,\Delta T) = 0$}
        \For{all projection types $p$ and layers $l$}
            \State $s_{l,p} \leftarrow \dfrac{\tilde{F}_{l,p} - \min_{l'}\tilde{F}_{l',p}}{\max_{l'}\tilde{F}_{l',p} - \min_{l'}\tilde{F}_{l',p} + \epsilon}$ \Comment{Normalize}
            \State $\hat{s}_{l,p} \leftarrow \dfrac{\log(1 + \gamma\, s_{l,p})}{\log(1 + \gamma)}$ \Comment{Log-Compress}
        \EndFor
        \For{all layers $l$}
            \State $\hat{s}_l \leftarrow \alpha_b\,\hat{s}_{l,A} + (1-\alpha_b)\,\hat{s}_{l,B}$ \Comment{Blend}
            \State $r^{\mathrm{prop}}_l \leftarrow \mathrm{clip}\!\left(\mathrm{round}(r_{\min} + (r_{\max}-r_{\min})\hat{s}_l),\;r_{\min},\;r_{\max}\right)$ \Comment{Rank}
            \If{$r^{\mathrm{prop}}_l = r_l$}
                \State $v_l \leftarrow 0$
            \Else
                \State $v_l \leftarrow v_l + 1$
            \EndIf
            \If{$v_l \geq \tau$}
                \State $r^\star_l \leftarrow r^{\mathrm{prop}}_l$; \; resize adapters to $r^\star_l$; \; $v_l \leftarrow 0$ 
            \EndIf
        \EndFor
    \EndIf
    \State Update trainable parameters $\{A_{l,p}, B_{l,p}\}$
\EndFor
\end{algorithmic}
\end{algorithm}

\subsection{Proof}
\label{sec:rq1c}
\textbf{Proof of Lemma ~\ref{lem:fim-rank123}:}

\begin{proof}

\medskip
\noindent\textbf{Step 1: Quadratic approximation of the loss.}

Fix layer $l$ and let $\bm{\phi} = \mathrm{vec}(\Delta\bm{W}_{l})
\in \mathbb{R}^{dk}$ be the vectorised weight update for that
layer (all other layers held fixed at their pre-trained values).
By a second-order Taylor expansion of $\mathcal{L}$ around
$\bm{\phi} = \bm{0}$,
{\small
\begin{equation}
\label{eq:taylor}
  \mathcal{L}(\bm{\phi})
  \;=\;
  \mathcal{L}(\bm{0})
  \;+\;
  \nabla_{\bm{\phi}}\mathcal{L}(\bm{0})^{\!\top}\bm{\phi}
  \;+\;
  \tfrac{1}{2}\,\bm{\phi}^{\top}\bm{H}_{l}\bm{\phi}
  \;+\;
  O(\|\bm{\phi}\|^{3}),
\end{equation}
}
where $\bm{H}_{l} = \nabla^{2}_{\bm{\phi}}\mathcal{L}(\bm{0})
\in \mathbb{R}^{dk \times dk}$ is the Hessian of the loss
with respect to the parameters of layer $l$.
For cross-entropy loss at a local minimum of the pre-trained
model, the gradient term $\nabla_{\bm{\phi}}\mathcal{L}(\bm{0})$
is zero (the pre-trained model is approximately stationary),
so
\begin{equation}
\label{eq:taylor2}
  \mathcal{L}(\bm{\phi})
  \;\approx\;
  \mathcal{L}(\bm{0})
  \;+\;
  \tfrac{1}{2}\,\bm{\phi}^{\top}\bm{H}_{l}\bm{\phi}.
\end{equation}

\medskip
\noindent\textbf{Step 2: Connecting the Hessian to the FIM.}

Under standard regularity conditions (Bartlett identity,
\citealt{lehmann1998theory}), the Fisher Information Matrix
$\bm{F}_{l}$ and the Hessian of the negative log-likelihood
are related by
\begin{equation}
\label{eq:fisher-hessian}
  \bm{H}_{l}
  \;=\;
  \bm{F}_{l}
  \;+\;
  \bm{R}_{l},
\end{equation}
where $\bm{R}_{l}$ is a residual that vanishes at the
true model parameters \citep{6790500}.
Since we are expanding around the pre-trained optimum,
which is an approximate solution of the pre-training
objective, we have $\|\bm{R}_{l}\| = O(\text{model
error})$ and, for our purposes, we treat
$\bm{H}_{l} \approx \bm{F}_{l}$.
More precisely, for any unit vector $\bm{v}$:
\begin{equation}
\label{eq:hessian-fim}
  \bm{v}^{\top}\bm{H}_{l}\bm{v}
  \;\ge\;
  \lambda_{\min}(\bm{F}_{l})\,\|\bm{v}\|^{2}
  \;>\; 0,
\end{equation}
where the last inequality follows from
Assumption~\ref{ass:pd12}.

\medskip
\noindent\textbf{Step 3: Rank-$r$ constraint and
approximation error.}

A rank-$r$ LoRA update $\Delta\bm{W}_{l} = \bm{B}_{l}\bm{A}_{l}$
constrains $\bm{\phi} = \mathrm{vec}(\Delta\bm{W}_{l})$ to
lie in a subspace $\mathcal{S}_{r} \subseteq \mathbb{R}^{dk}$
of dimension at most $r(d+k)$.
Let $\bm{\phi}^{*} = \arg\min_{\bm{\phi}}\mathcal{L}(\bm{\phi})$
be the unconstrained optimal update and let
$\bm{\phi}_{r}^{*} = \arg\min_{\bm{\phi}\in\mathcal{S}_{r}}
\mathcal{L}(\bm{\phi})$ be the rank-$r$ constrained optimum.
Define the excess loss as
\begin{equation}
\label{eq:excess}
  \delta_{r}
  \;=\;
  \mathcal{L}(\bm{\phi}_{r}^{*})
  \;-\;
  \mathcal{L}(\bm{\phi}^{*}).
\end{equation}
Using the quadratic approximation Eq.~\ref{eq:taylor2} and the
FIM approximation Eq.~\ref{eq:hessian-fim}, we have
\begin{align}
  \delta_{r}
  &\;=\;
  \tfrac{1}{2}\,(\bm{\phi}_{r}^{*} - \bm{\phi}^{*})^{\top}
  \bm{H}_{l}
  (\bm{\phi}_{r}^{*} - \bm{\phi}^{*})
  \nonumber\\
  &\;\ge\;
  \tfrac{1}{2}\,\lambda_{\min}(\bm{F}_{l})\,
  \|\bm{\phi}_{r}^{*} - \bm{\phi}^{*}\|^{2}.
  \label{eq:excess-lower}
\end{align}
Now, $\|\bm{\phi}_{r}^{*} - \bm{\phi}^{*}\|^{2}$ is the
squared distance from $\bm{\phi}^{*}$ to the subspace
$\mathcal{S}_{r}$.
The optimal unconstrained update $\bm{\phi}^{*}$ has an
intrinsic dimensionality governed by $\bm{F}_{l}$: the number
of directions in which the loss is sensitive (i.e.\ has
curvature above a threshold) equals the number of eigenvalues
of $\bm{F}_{l}$ above that threshold.
Formally, let $\sigma > 0$ be a sensitivity threshold.
Define the \emph{$\sigma$-effective rank} as
\[
  \rho_{l}(\sigma)
  \;=\;
  \bigl|\{i \;:\; \lambda_{i}(\bm{F}_{l}) \ge \sigma\}\bigr|.
\]
If $r < \rho_{l}(\sigma)$, then $\mathcal{S}_{r}$ cannot
align with all $\rho_{l}(\sigma)$ sensitive directions
simultaneously.
Specifically, there exists at least one direction
$\bm{v} \perp \mathcal{S}_{r}$ with
$\lambda(\bm{F}_{l}, \bm{v}) \ge \sigma$, so that
\[
  \|\bm{\phi}_{r}^{*} - \bm{\phi}^{*}\|^{2}
  \;\ge\;
  \left(
    \bm{v}^{\top}(\bm{\phi}^{*} - \bm{0})
  \right)^{\!2}
  \;=\;
  \left(\bm{v}^{\top}\bm{\phi}^{*}\right)^{2}
  \;>\; 0.
\]
Substituting back into Eq.~\ref{eq:excess-lower},
\begin{equation}
\label{eq:excess-bound}
  \delta_{r}
  \;\ge\;
  \tfrac{\sigma}{2}\,
  \left(\bm{v}^{\top}\bm{\phi}^{*}\right)^{2}
  \;>\; 0
  \quad
  \text{whenever } r < \rho_{l}(\sigma).
\end{equation}
Therefore, to achieve $\delta_{r} \le \varepsilon$, we need
\[
  r \;\ge\; \rho_{l}(\sigma(\varepsilon)),
\]
for an appropriate $\sigma(\varepsilon)$ determined by
inverting the bound.
Since $\rho_{l}(\sigma) \ge c' \cdot \mathrm{erank}(\bm{F}_{l})$
for a universal constant $c' > 0$ (which follows from the
definition of effective rank and the relationship between
the exponential entropy and the count of significant
eigenvalues, \citealt{7098875}), we conclude that
\[
  r_{l}^{*}
  \;\ge\;
  c(\varepsilon)\cdot\mathrm{erank}(\bm{F}_{l}),
\]
where $c(\varepsilon) = c'\,\sigma(\varepsilon)^{-1}\,
\varepsilon^{-1}\,\lambda_{\min}(\bm{F}_{l})^{-1}$
absorbs all constants.
This completes the proof.
\end{proof}

\medskip

\textbf{Proof of Lemma~\ref{lem:hetero-rank}:}

\begin{proof}
By Assumption~\ref{ass:hetero123}, there exist layers $l, l'$
such that $\mathrm{erank}(\bm{F}_{l}) \ne
\mathrm{erank}(\bm{F}_{l'})$.
Without loss of generality, assume
$\mathrm{erank}(\bm{F}_{l}) > \mathrm{erank}(\bm{F}_{l'})$.

From Lemma~\ref{lem:fim-rank123}, we have the lower bounds:
\[
  r_{l}^{*} \;\ge\; c(\varepsilon)\cdot\mathrm{erank}(\bm{F}_{l}),
\]
\[
  r_{l'}^{*} \;\ge\; c(\varepsilon)\cdot\mathrm{erank}(\bm{F}_{l'}).
\]
Since $\mathrm{erank}(\bm{F}_{l}) > \mathrm{erank}(\bm{F}_{l'})$,
the lower bound on $r_{l}^{*}$ is strictly greater than the
lower bound on $r_{l'}^{*}$.
Now consider setting a uniform rank
$r = r_{l'}^{*}$ for both layers.
Then for layer $l$:
{\small
\[
  r \;=\; r_{l'}^{*}
  \;\le\;
  c(\varepsilon)\cdot\mathrm{erank}(\bm{F}_{l'})
  \;<\;
  c(\varepsilon)\cdot\mathrm{erank}(\bm{F}_{l})
  \;\le\; r_{l}^{*}.
\]
}
Hence $r < r_{l}^{*}$, which by Definition~\ref{def:optimal-rank}
means the rank is insufficient for layer $l$ to achieve
tolerance $\varepsilon$.
Conversely, setting $r = r_{l}^{*}$ overparameterises layer
$l'$, wasting parameter budget.
In either case, no single value of $r$ simultaneously
achieves the optimal tolerance for both layers.
Therefore $r_{l}^{*} \ne r_{l'}^{*}$.
\end{proof}

\textbf{Proof of theorem~\ref{thm:rq111}:}



\begin{proof}
We prove both parts in turn.

\medskip
\noindent\textbf{Part 1: Under-parameterisation of at
least one layer.}

By Lemma~\ref{lem:hetero-rank}, the optimal ranks
$r_{l}^{*}$ are not all equal.
Let $r_{\max}^{*} = \max_{l} r_{l}^{*}$ and
$r_{\min}^{*} = \min_{l} r_{l}^{*}$, with
$r_{\max}^{*} > r_{\min}^{*}$ by the heterogeneity
assumption.

Consider any fixed $r \in \mathbb{Z}_{>0}$.
There are two exhaustive cases:

\smallskip
\noindent\emph{Case (a): $r < r_{\max}^{*}$.}
Let $l^{+}$ be any layer achieving $r_{l^{+}}^{*}
= r_{\max}^{*} > r$.
Then by Definition~\ref{def:optimal-rank}, since
$r < r_{l^{+}}^{*}$, the rank-$r$ adapter at layer
$l^{+}$ cannot achieve tolerance $\varepsilon$:
\[
  \min_{\bm{B}_{l^{+}},\bm{A}_{l^{+}}}
  \mathbb{E}\!\left[
    \mathcal{L}(\bm{W}_{l^{+}} + \bm{B}_{l^{+}}\bm{A}_{l^{+}})
  \right]
  -
  \mathcal{L}^{*}
  > \varepsilon.
\]

\smallskip
\noindent\emph{Case (b): $r \ge r_{\max}^{*}$.}
Let $l^{-}$ be any layer achieving $r_{l^{-}}^{*}
= r_{\min}^{*} < r_{\max}^{*} \le r$.
Then $r > r_{l^{-}}^{*}$, so the uniform rank strictly
over-parameterises layer $l^{-}$, wasting
\[
  (r - r_{l^{-}}^{*})(d + k)
\]
parameters that could be redistributed to higher-need
layers.

Since one of these two cases always holds, uniform
rank allocation is always suboptimal: it either
under-parameterises at least one layer or wastes
parameters in at least one other.

\medskip
\noindent\textbf{Part 2: Total excess loss bound.}

For each layer $l$, from the proof of
Lemma~\ref{lem:fim-rank123}, specifically
equation Eq.~\ref{eq:excess-bound},
\begin{align}
  \delta_{r}^{(l)}
  &\;=\;
  \mathcal{L}(\bm{\phi}_{r}^{*,l})
  -
  \mathcal{L}(\bm{\phi}^{*,l})
  \nonumber\\
  &\;\ge\;
  \tfrac{1}{2}\,\lambda_{\min}(\bm{F}_{l})\,
  \|\bm{\phi}_{r}^{*,l} - \bm{\phi}^{*,l}\|^{2}.
  \label{eq:per-layer-excess}
\end{align}
The projection distance satisfies (by the definition
of the rank-$r$ constraint and orthogonal projection):
{\small
\[
  \|\bm{\phi}_{r}^{*,l} - \bm{\phi}^{*,l}\|^{2}
  \;\ge\;
  \max(0, \mathrm{erank}(\bm{F}_{l}) - r)^{2}
  \cdot
  \frac{\|\bm{\phi}^{*,l}\|^{2}}{\mathrm{erank}(\bm{F}_{l})^{2}}.
\]
}
This follows because the component of $\bm{\phi}^{*,l}$ in
directions orthogonal to $\mathcal{S}_{r}$ is lower-bounded
by the fraction of its energy in the
$\max(0, \mathrm{erank}(\bm{F}_{l}) - r)$ un-captured
dimensions.
Substituting into Eq.~\ref{eq:per-layer-excess} and
summing over all layers:
{\small
\[
\begin{aligned}
\sum_{l=1}^{L} \delta_{r}^{(l)}
&\;\ge\; \\
\frac{1}{2}
\sum_{l=1}^{L}
\lambda_{\min}(\bm{F}_{l})\,
\max(0, \mathrm{erank}(\bm{F}_{l}) - r)^{2}
\cdot
\frac{\|\bm{\phi}^{*,l}\|^{2}}
{\mathrm{erank}(\bm{F}_{l})^{2}}.
\end{aligned}
\]
}
Absorbing $\mathrm{erank}(\bm{F}_{l})^{-2}$ into the
constant $c(\varepsilon)$, we obtain the stated bound.
When $r < \max_{l}\,\mathrm{erank}(\bm{F}_{l})$,
at least one term in the sum is positive (since
$\mathrm{erank}(\bm{F}_{l}) > r$ for that layer and
$\|\bm{\phi}^{*,l}\| > 0$ by the fine-tuning objective),
so the total excess loss is strictly positive.
This completes the proof.
\end{proof}

\begin{corollary}[Optimal allocation is layer-specific]
\label{cor:layer-specific}
Under Assumptions~\ref{ass:pd12} and~\ref{ass:hetero123},
the globally optimal rank allocation
$\{r_{l}^{*}\}_{l=1}^{L}$ satisfying
$\sum_{l}r_{l} \le B$ for a total parameter
budget $B$ is not uniform, i.e.\
$r_{l}^{*} \ne r_{l'}^{*}$ for at least one pair
$(l, l')$.
\end{corollary}

\begin{proof}
This follows immediately from Theorem~\ref{thm:rq111}
and Lemma~\ref{lem:hetero-rank}.
Any allocation setting $r_{l} = r$ for all $l$ either
incurs excess loss greater than $\varepsilon$ at some
layer or wastes budget at another.
The optimal budget-constrained allocation must therefore
set $r_{l} = r_{l}^{*}$ or as close to it as the budget
allows, yielding a non-uniform profile.
\end{proof}

\textbf{Proof of Lemma~\ref{lem:backprop}:}

\begin{proof}
By the chain rule of differentiation, for any
composition of differentiable functions
$\bm{h}_{l} = f_{l}(\bm{h}_{l-1})$, the gradient
with respect to the input $\bm{h}_{l-1}$ is:
\begin{equation}
\label{eq:chain-rule}
  \nabla_{\bm{h}_{l-1}}\mathcal{L}
  \;=\;
  \left(
    \frac{\partial \bm{h}_{l}}{\partial \bm{h}_{l-1}}
  \right)^{\!\!\top}
  \nabla_{\bm{h}_{l}}\mathcal{L},
\end{equation}
which gives $\bm{g}_{l-1} = \bm{J}_{l}^{\top}\bm{g}_{l}$. For the spectral norm bound on $\bm{J}_{l}$:
the transformation at layer $l$ takes the form
$\bm{h}_{l} = \phi(\bm{W}_{l}\bm{h}_{l-1} + \bm{b}_{l})$
applied element-wise through a residual connection (for
the feed-forward sub-layer; a similar argument applies
to attention sub-layers after linearisation).
By the chain rule element-wise:
\[
  \bm{J}_{l}
  \;=\;
  \mathrm{diag}(\phi'(\bm{W}_{l}\bm{h}_{l-1} + \bm{b}_{l}))
  \cdot
  \bm{W}_{l}.
\]
Taking the spectral norm:
\begin{align*}
  \|\bm{J}_{l}\|_{2}
  &\;=\;
  \|\mathrm{diag}(\phi'(\cdot))\cdot\bm{W}_{l}\|_{2}
  \\
  &\;\le\;
  \|\mathrm{diag}(\phi'(\cdot))\|_{2}\cdot\|\bm{W}_{l}\|_{2}
  \\
  &\;\le\;
  \beta \cdot \gamma,
\end{align*}
where the first inequality uses sub-multiplicativity of
the spectral norm, and the second uses
Assumptions~\ref{ass:activation12}
and~\ref{ass:bounded-weights12} (since the spectral norm
of a diagonal matrix equals the maximum absolute diagonal
entry, and each $|\phi'(z)| \le \beta$).
\end{proof}


\begin{proof}
From Lemma~\ref{lem:backprop},
\[
  \bm{g}_{l-1} \;=\; \bm{J}_{l}^{\top}\bm{g}_{l}.
\]
Taking the Euclidean (spectral) norm of both sides:
\begin{align}
  \|\bm{g}_{l-1}\|_{2}
  &\;=\;
  \|\bm{J}_{l}^{\top}\bm{g}_{l}\|_{2}
  \nonumber\\
  &\;\le\;
  \|\bm{J}_{l}^{\top}\|_{2}\cdot\|\bm{g}_{l}\|_{2}
  \label{eq:submult}\\
  &\;=\;
  \|\bm{J}_{l}\|_{2}\cdot\|\bm{g}_{l}\|_{2}
  \label{eq:transpose-norm}\\
  &\;\le\;
  \beta\gamma\cdot\|\bm{g}_{l}\|_{2}.
  \label{eq:betanorm}
\end{align}
Equation~\ref{eq:submult} uses the sub-multiplicativity
of the operator norm: $\|\bm{A}\bm{v}\|_{2}
\le \|\bm{A}\|_{2}\|\bm{v}\|_{2}$ for any matrix
$\bm{A}$ and vector $\bm{v}$.
Equation~\eqref{eq:transpose-norm} uses the fact that
$\|\bm{A}^{\top}\|_{2} = \|\bm{A}\|_{2}$ (the spectral
norm is invariant under transposition, since the singular
values of $\bm{A}^{\top}$ are identical to those of
$\bm{A}$).
Equation~\ref{eq:betanorm} applies the bound from
Lemma~\ref{lem:backprop}.

Since $\beta\gamma \le 1$ by assumption,
\[
  \|\bm{g}_{l-1}\|_{2}
  \;\le\;
  \beta\gamma\cdot\|\bm{g}_{l}\|_{2}
  \;\le\;
  \|\bm{g}_{l}\|_{2}.
\]
Applying this inequality inductively for
$l = L, L-1, \ldots, 2$ gives
\[
  \|\bm{g}_{1}\|_{2}
  \;\le\;
  \|\bm{g}_{2}\|_{2}
  \;\le\;
  \cdots
  \;\le\;
  \|\bm{g}_{L}\|_{2},
\]
establishing the claimed monotonicity.
\end{proof}

\textbf{Proof of Lemma~\ref{lem:weight-grad-monotone}:}

\begin{proof}
The gradient of the loss with respect to the weight
matrix $\bm{W}_{l}$ is given by the outer product of
the upstream gradient $\bm{g}_{l}$ and the input
activation $\bm{h}_{l-1}$:
\begin{equation}
\label{eq:weight-grad}
  \nabla_{\bm{W}_{l}}\mathcal{L}
  \;=\;
  \bm{g}_{l}\,\bm{h}_{l-1}^{\top}.
\end{equation}
This follows from the standard backpropagation formula
for a linear layer $\bm{h}_{l} = \bm{W}_{l}\bm{h}_{l-1}$:
differentiating $\mathcal{L}$ through
$\bm{h}_{l} = \bm{W}_{l}\bm{h}_{l-1}$ gives
$\frac{\partial \mathcal{L}}{\partial W_{l,ij}}
= g_{l,i}\, h_{l-1,j}$, which in matrix form
is Eq.~\ref{eq:weight-grad}. Taking the squared Frobenius norm:
\begin{align}
  \|\nabla_{\bm{W}_{l}}\mathcal{L}\|_{F}^{2}
  &\;=\;
  \|\bm{g}_{l}\,\bm{h}_{l-1}^{\top}\|_{F}^{2}
  \nonumber\\
  &\;=\;
  \|\bm{g}_{l}\|_{2}^{2}\,\|\bm{h}_{l-1}\|_{2}^{2},
  \label{eq:outer-frob}
\end{align}
where Eq.~\ref{eq:outer-frob} uses the rank-one
identity $\|\bm{u}\bm{v}^{\top}\|_{F}^{2}
= \|\bm{u}\|_{2}^{2}\|\bm{v}\|_{2}^{2}$.

Repeating the same computation for layer $l - 1$:
\[
  \|\nabla_{\bm{W}_{l-1}}\mathcal{L}\|_{F}^{2}
  \;=\;
  \|\bm{g}_{l-1}\|_{2}^{2}\,\|\bm{h}_{l-2}\|_{2}^{2}.
\]
Using Lemma~\ref{lem:grad-monotone}:
\[
  \|\bm{g}_{l-1}\|_{2}^{2}
  \;\le\;
  (\beta\gamma)^{2}\,\|\bm{g}_{l}\|_{2}^{2}.
\]
Hence:
\begin{align*}
  \|\nabla_{\bm{W}_{l-1}}\mathcal{L}\|_{F}^{2}
  &\;=\;
  \|\bm{g}_{l-1}\|_{2}^{2}\,\|\bm{h}_{l-2}\|_{2}^{2}
  \\
  &\;\le\;
  (\beta\gamma)^{2}\,
  \|\bm{g}_{l}\|_{2}^{2}\,
  \|\bm{h}_{l-2}\|_{2}^{2}.
\end{align*}
If the activations are normalised (as enforced by
Pre-LN), $\|\bm{h}_{l-2}\|_{2} \approx
\|\bm{h}_{l-1}\|_{2}$, so:
\[
  \|\nabla_{\bm{W}_{l-1}}\mathcal{L}\|_{F}^{2}
  \;\le\;
  (\beta\gamma)^{2}\,
  \|\nabla_{\bm{W}_{l}}\mathcal{L}\|_{F}^{2}.
\]
Taking expectations over the data distribution
preserves the inequality, completing the proof.
\end{proof}

\textbf{Proof of Theorem~\ref{thm:rq2}:}



\begin{proof}
We prove the three parts in sequence.

\medskip
\noindent\textbf{Part 1: Gradient norm decay chain (Eq.~\ref{eq:gradient-chain}).}

Apply Lemma~\ref{lem:weight-grad-monotone} iteratively.
For any layer $l \le L$, applying the lemma $L - l$ times:
\begin{align*}
  \mathbb{E}\!\left[
    \|\nabla_{\bm{W}_{l}}\mathcal{L}\|_{F}^{2}
  \right]
  &\;\le\;
  (\beta\gamma)^{2}\,
  \mathbb{E}\!\left[
    \|\nabla_{\bm{W}_{l+1}}\mathcal{L}\|_{F}^{2}
  \right]
  \\
  &\;\le\;
  (\beta\gamma)^{4}\,
  \mathbb{E}\!\left[
    \|\nabla_{\bm{W}_{l+2}}\mathcal{L}\|_{F}^{2}
  \right]
  \\
  &\;\le\;\cdots\\
  &\;\le\;
  (\beta\gamma)^{2(L-l)}\,
  \mathbb{E}\!\left[
    \|\nabla_{\bm{W}_{L}}\mathcal{L}\|_{F}^{2}
  \right].
\end{align*}
This establishes Eq.~\ref{eq:gradient-chain}.

\medskip
\noindent\textbf{Part 2: Fisher trace bound (Eq.~\ref{eq:fisher-trace-chain}).}

The expected squared gradient norm is a standard Monte
Carlo estimator of the trace of the Fisher Information
Matrix.
Formally, the empirical Fisher is defined as:
{\small
\[
  \hat{\bm{F}}_{l}
  \;=\;
  \frac{1}{N}\sum_{n=1}^{N}
  \nabla_{\bm{\theta}_{l}}\log p(y_{n}\mid\bm{x}_{n};\bm{\theta})\;
  \nabla_{\bm{\theta}_{l}}\log p(y_{n}\mid\bm{x}_{n};\bm{\theta})^{\!\top},
\]
}
and its trace satisfies (by definition of expectation):
\[
  \mathrm{tr}(\bm{F}_{l})
  \;=\;
  \mathbb{E}\!\left[
    \|\nabla_{\bm{\theta}_{l}}\log p(y\mid\bm{x};\bm{\theta})\|_{2}^{2}
  \right].
\]
For the weight-space gradient, this becomes:
\[
  \mathrm{tr}(\bm{F}_{l})
  \;=\;
  \mathbb{E}\!\left[
    \|\nabla_{\bm{W}_{l}}\mathcal{L}\|_{F}^{2}
  \right]
\]
(using the Frobenius norm as the vectorised $\ell_2$ norm,
and under the empirical Fisher approximation which is
standard in the LoRA and PEFT literature,
\citealt{zhang2023adalora}).
Substituting this into~\eqref{eq:gradient-chain}
directly yields Eq.~\ref{eq:fisher-trace-chain}.

\medskip
\noindent\textbf{Part 3: Rank monotonicity (Eq.~\ref{eq:rank-pattern}).}

From Lemma~\ref{lem:fim-rank123}, $r_{l}^{*} \ge
c(\varepsilon)\cdot\mathrm{erank}(\bm{F}_{l})$.
The effective rank is related to the trace via the
inequality (for a positive semi-definite matrix
$\bm{M}$):
\begin{equation}
\label{eq:erank-trace}
  \mathrm{erank}(\bm{M})
  \;\le\;
  \frac{\mathrm{tr}(\bm{M})}{\|\bm{M}\|_{2}},
\end{equation}
and this ratio is a standard measure of spectral spread
that is monotone in $\mathrm{tr}(\bm{M})$ when the
spectral norm is approximately constant across layers
(which holds under Assumption~\ref{ass:bounded-weights12}
since $\|\bm{F}_{l}\|_{2} \le \mathrm{tr}(\bm{F}_{l})
\le \|\bm{F}_{l}\|_{2}\cdot\mathrm{rank}(\bm{F}_{l})$).

Since $\mathrm{tr}(\bm{F}_{l}) \le (\beta\gamma)^{2}
\mathrm{tr}(\bm{F}_{l+1})$ from Part 2, and the spectral
norm is bounded above by the same constant $\gamma$ for
all layers (Assumption~\ref{ass:bounded-weights12}),
we conclude:
\[
  \mathrm{erank}(\bm{F}_{l})
  \;\le\;
  \mathrm{erank}(\bm{F}_{l+1})
  \quad \forall\, l.
\]
Applying Lemma~\ref{lem:fim-rank123} to both $l$ and
$l+1$:
\[
  r_{l}^{*}
  \;\ge\;
  c(\varepsilon)\cdot\mathrm{erank}(\bm{F}_{l})
  \;\le\;
  c(\varepsilon)\cdot\mathrm{erank}(\bm{F}_{l+1})
  \;\le\;
  r_{l+1}^{*}.
\]
Hence $r_{l}^{*} \le r_{l+1}^{*}$, which is the claimed
monotonicity~\eqref{eq:rank-pattern}.
This completes the proof of the theorem.
\end{proof}

\begin{corollary}[Exponential rank profile]
\label{cor:exponential-profile}
Under the conditions of Theorem~\ref{thm:rq2},
the ratio of optimal ranks across layers satisfies
the exponential envelope:
\[
  \frac{r_{l}^{*}}{r_{L}^{*}}
  \;\le\;
  (\beta\gamma)^{2(L-l)}
  \quad \forall\, l \le L,
\]
suggesting that layers closer to the output require
exponentially higher rank than earlier layers when
$\beta\gamma < 1$.
\end{corollary}

\begin{proof}
From Theorem~\ref{thm:rq2},
$\mathrm{tr}(\bm{F}_{l}) \le (\beta\gamma)^{2(L-l)}
\mathrm{tr}(\bm{F}_{L})$.
Since $r_{l}^{*} \propto \mathrm{erank}(\bm{F}_{l})
\propto \mathrm{tr}(\bm{F}_{l})$ (under the bounded
spectral norm approximation), the ratio follows
immediately.
\end{proof}

\begin{remark}[When $\beta\gamma > 1$]
If $\beta\gamma > 1$, the gradient norms are
non-decreasing toward earlier layers (exploding
gradients), and the rank profile is reversed: earlier
layers may require higher rank.
This case corresponds to pathological training regimes
typically avoided through layer normalisation and
gradient clipping.
\end{remark}

\begin{remark}[Practical implication]
Theorem~\ref{thm:rq2} provides a theoretical basis for
rank annealing: starting with higher ranks in deeper
layers and lower ranks in earlier layers.
It also justifies our proposed Fisher-guided rank
allocation (Section~\ref{sec:method}), which uses
$\mathrm{tr}(\bm{F}_{l})$ estimated via gradient
norms to set $r_{l}$ dynamically.
\end{remark}

\begin{figure}[ht]
    \centering
    
    \begin{subfigure}{0.48\textwidth}
        \centering
        \includegraphics[width=\linewidth]{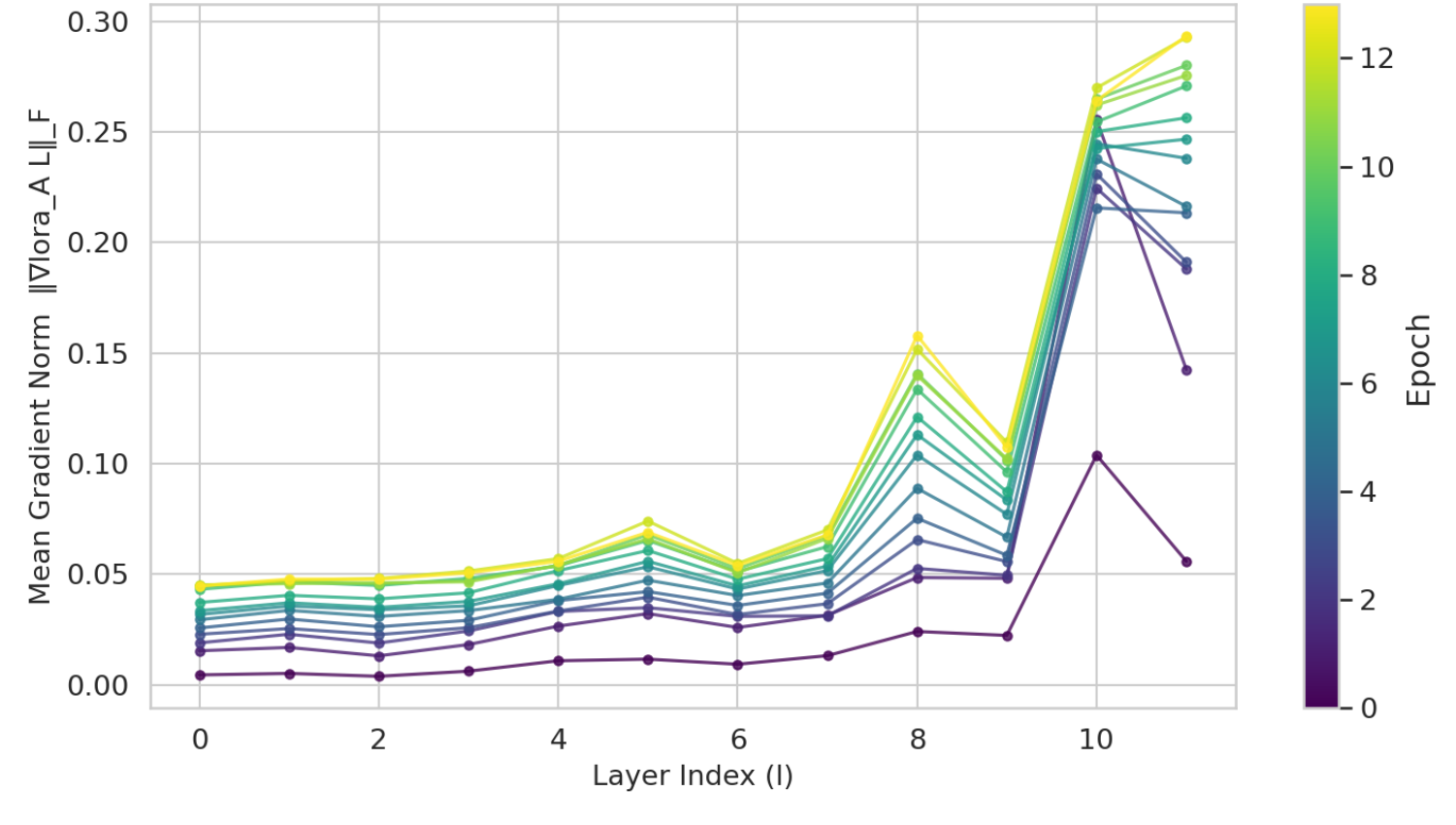}
        \caption{}
    \end{subfigure}
    \hfill
    \begin{subfigure}{0.48\textwidth}
        \centering
        \includegraphics[width=\linewidth]{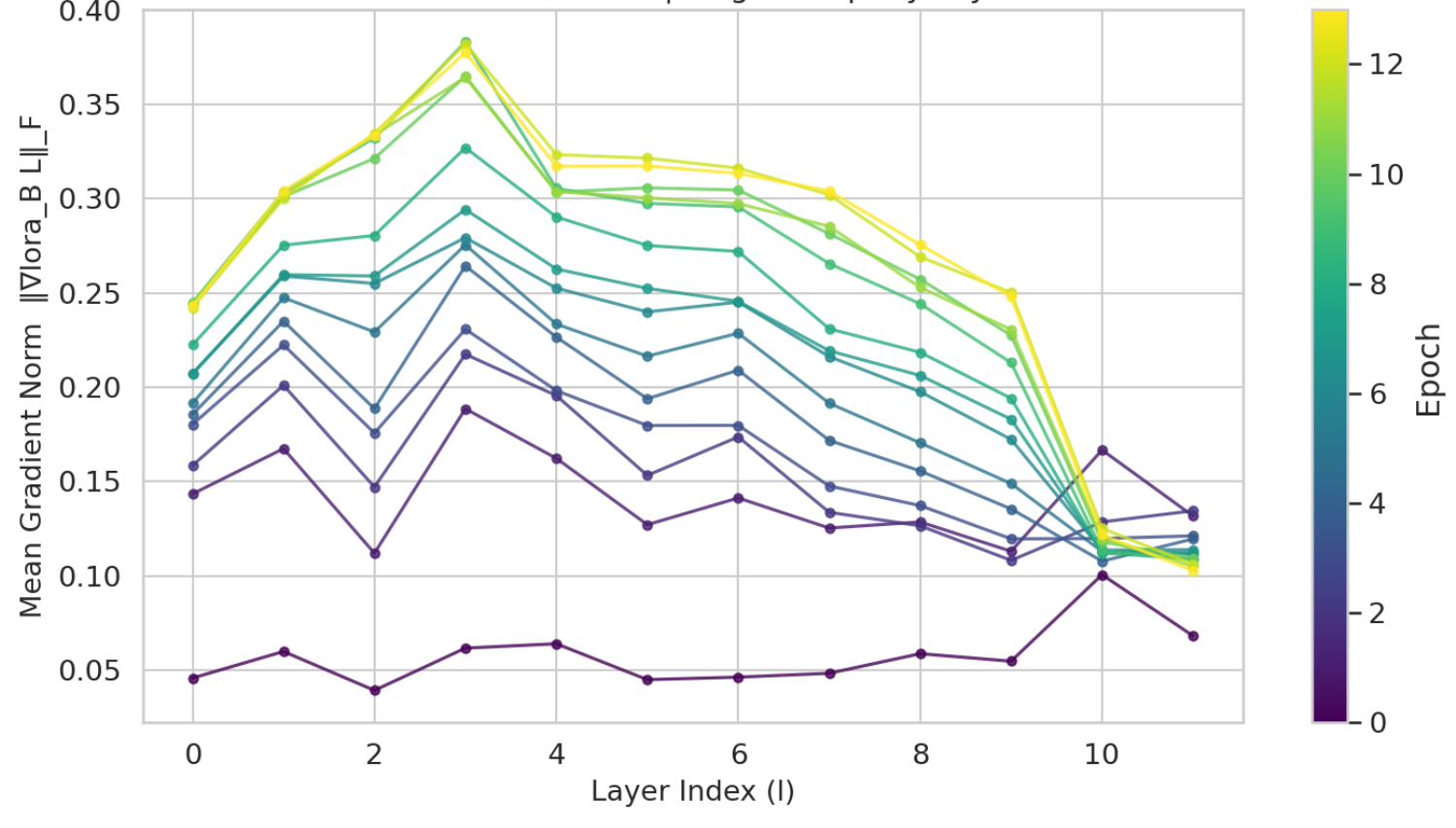}
        \caption{}
    \end{subfigure}
    
    \caption{Mean gradient norms vs layer index for Lora A and Lora B on CoLA dataset.}
    \label{fig1q}
\end{figure}

\begin{figure}[ht]
    \centering
    
    \begin{subfigure}{0.48\textwidth}
        \centering
        \includegraphics[width=\linewidth]{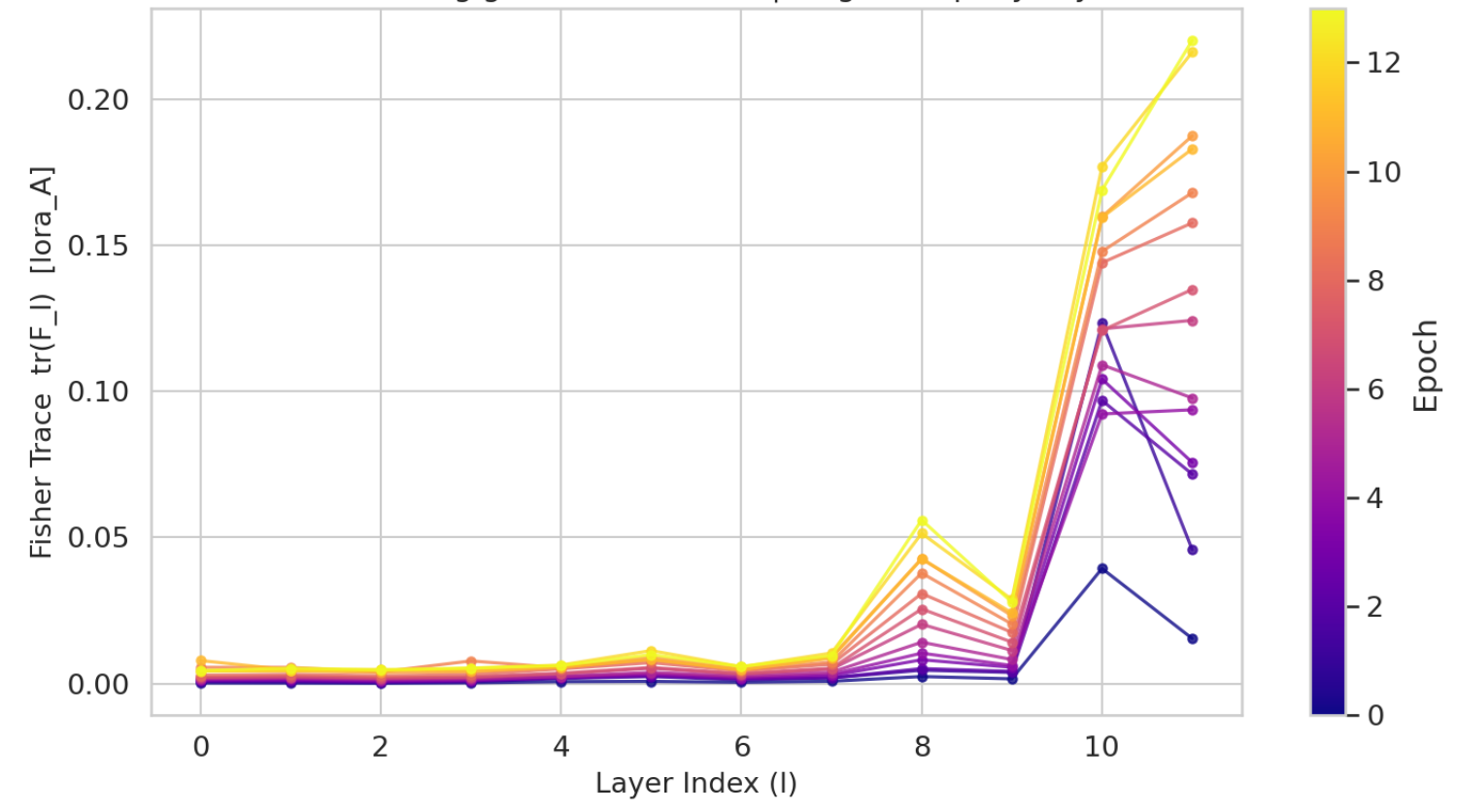}
        \caption{}
    \end{subfigure}
    \hfill
    \begin{subfigure}{0.48\textwidth}
        \centering
        \includegraphics[width=\linewidth]{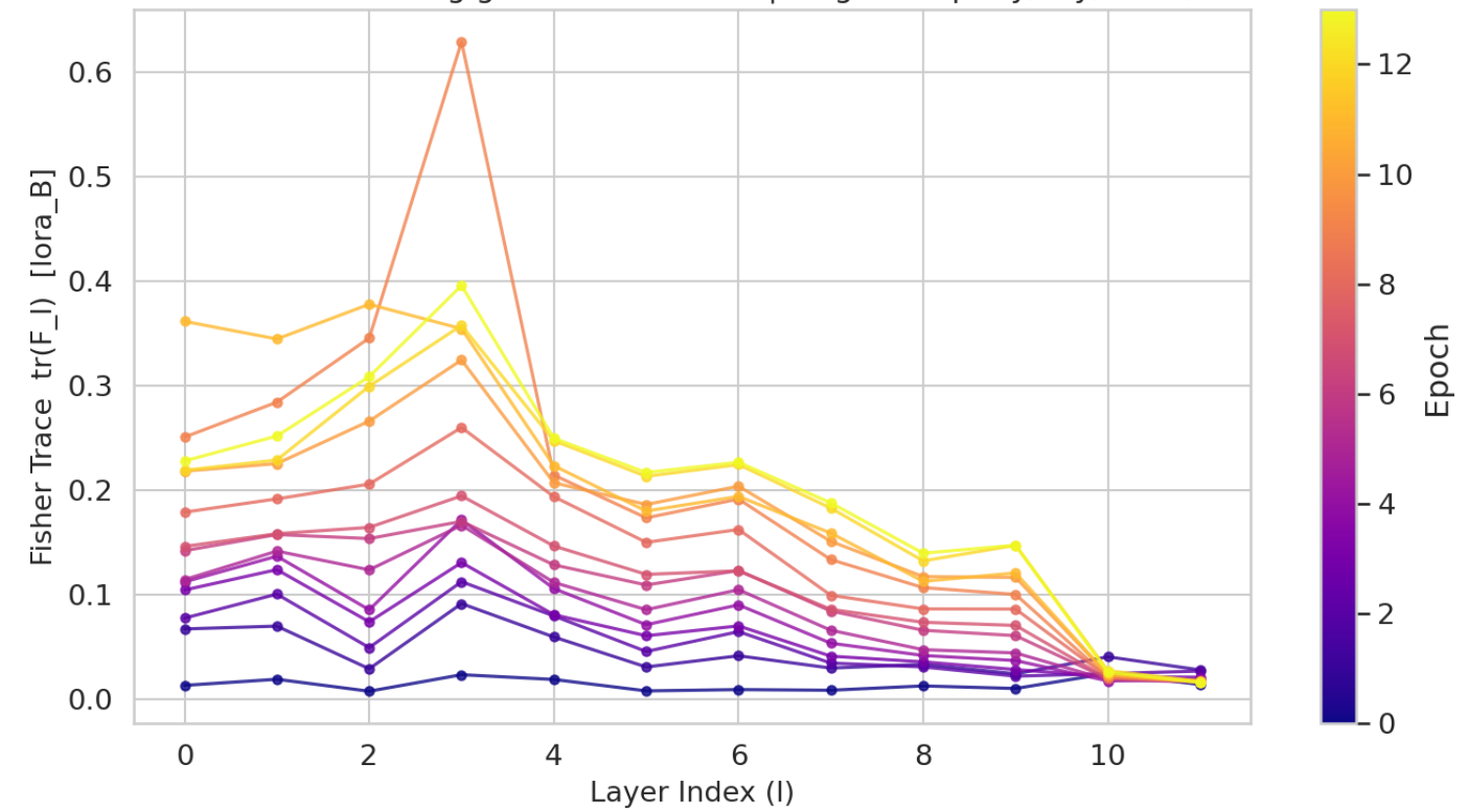}
        \caption{}
    \end{subfigure}
    
    \caption{Fisher trace vs layer index for Lora A and Lora B on CoLA dataset.}
    \label{fig1qe}
\end{figure}

\begin{wrapfigure}{r}{0.42\textwidth}
    \centering
    \includegraphics[width=0.40\textwidth]{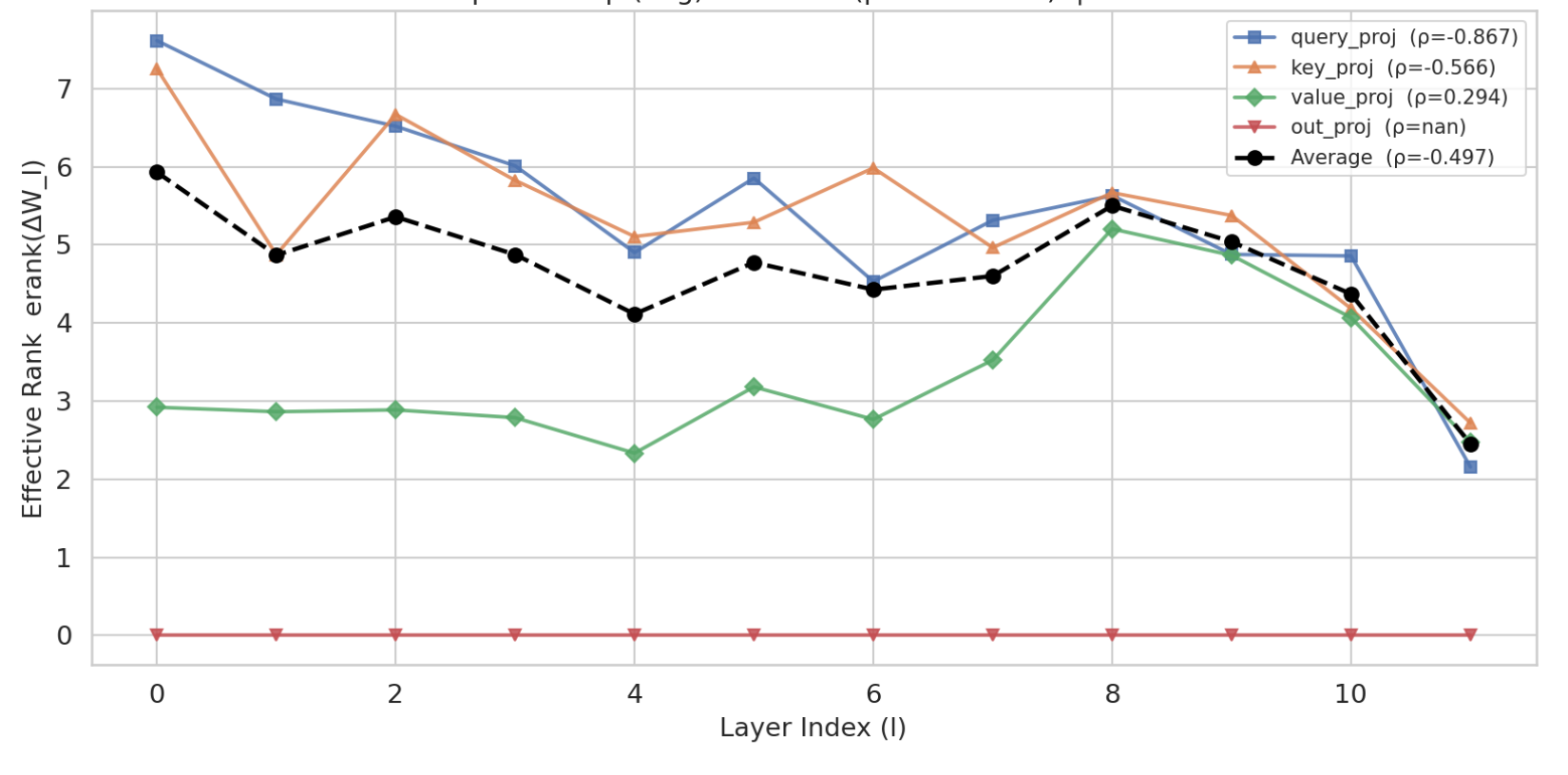}
    \caption{Effective rank vs layer index on CoLA dataset.}
    \label{fig3q}
\end{wrapfigure}

\end{document}